\pdfoutput=1
\documentclass[11pt]{article}

\usepackage[preprint]{acl}
\usepackage{times}
\usepackage{latexsym}
\usepackage[T1]{fontenc}
\usepackage[utf8]{inputenc}
\usepackage{microtype}
\usepackage{inconsolata}
\usepackage{graphicx}
\usepackage{dirtytalk}
\usepackage{placeins}
\usepackage{minitoc}
\usepackage{graphicx} 
\usepackage[utf8]{inputenc} 
\usepackage[T1]{fontenc}    
\usepackage{hyperref}       
\usepackage{url}            
\usepackage{booktabs}       
\usepackage{multirow}
\usepackage{amsfonts}       
\usepackage{nicefrac}       
\usepackage{microtype}      
\usepackage{xcolor}         
\usepackage{minitoc}
\usepackage{algorithm}
\usepackage{algpseudocode}
\usepackage{multicol}
\usepackage{tabularx}

\usepackage[arrow, matrix, curve]{xy}
\usepackage{calc,multirow,tikz,array,setspace,geometry}

\usepackage{graphicx}
\usepackage{algorithmicx}
\usepackage{algpseudocode}
\usepackage{algorithm}
\usepackage{subcaption}
\usepackage{tikz}
\usetikzlibrary{calc,positioning}
\usetikzlibrary{matrix,calc,positioning,arrows}
\usetikzlibrary{arrows,shapes}
\usetikzlibrary{quotes,arrows.meta}
\usepackage{hyperref}
\usepackage{tcolorbox}
\usepackage{textcomp}
\usepackage{color}
\usepackage{enumitem}
\usepackage{wrapfig}
\usepackage{array}
\usepackage{amsthm}
\usepackage{amssymb}
\usepackage{amsmath}
\usepackage{bbold}

\theoremstyle{plain}
\newtheorem{prop}{Proposition}

\newtheorem{lem}{Lemma}

\theoremstyle{definition}
\newtheorem{defn}{Definition}

\theoremstyle{remark}
\newtheorem{rem}{Remark}

\newcommand{\Hrond}{\mathcal H}
\newcommand{\Irond}{\mathcal I}

\newcommand{\Lrond}{\mathcal L}

\newcommand{\Srond}{\mathcal S}
\newcommand{\Trond}{\mathcal T}

\newcommand{\vertiii}[1]{{\left\vert\kern-0.25ex\left\vert\kern-0.25ex\left\vert #1 
    \right\vert\kern-0.25ex\right\vert\kern-0.25ex\right\vert}}

\newcommand{\set}[1]{\left\{ #1\right\}}

\renewcommand{\leq}{\leqslant}

\newcommand{\Prob}{\mathbb{P}}

\newcommand{\Xspace}{\mathsf{X}}

\usepackage{placeins}
\usepackage{bbm}

\definecolor{charteOrange}{RGB}{255,121,0}
\definecolor{charteYellow}{RGB}{255,220,0}
\definecolor{charteViolet}{RGB}{145,100,205}
\definecolor{chartePink}{RGB}{255,180,230}
\definecolor{charteGreen}{RGB}{80,190,135}
\definecolor{charteBlue}{RGB}{75,180,230}
\definecolor{charteGray}{RGB}{143,143,143}

\title{DivMerge: A divergence-based model merging method for multi-tasking}

\author{
\textbf{Brahim Touayouch\textsuperscript{1,2}}\quad
 \textbf{Lo{\"i}c Fosse\textsuperscript{1,3}}\quad
 \textbf{Géraldine Damnati\textsuperscript{1}}\quad
 \textbf{Gwénolé Lecorvé\textsuperscript{1}}\quad 
\\
 \textsuperscript{1}Orange Research, Lannion, France
 \\
 \textsuperscript{2}École polytechnique, Institut polytechnique de Paris, Palaiseau, France
 \\
 \textsuperscript{3}CNRS, LIS, Aix Marseille Université, France
 \\
 \textbf{Contact:} {\tt first.last@orange.com}
}

\begin{document}

\tikzstyle{fleche} = [draw, thick, color=black, -latex']
\tikzstyle{dfleche} = [draw, thick, color=black, latex'-latex']
\tikzstyle{flecheDashed} = [draw, dashed, color=black, -latex']
\tikzstyle{DflecheDashed} = [draw, dashed, color=black, latex'-latex']

\maketitle
\begin{abstract}
    Multi-task learning (MTL) is often achieved by merging datasets before fine-tuning, but the growing availability of fine-tuned models has led to new approaches such as model merging via task arithmetic. 
    A major challenge in this setting is task interference, which worsens as the number of tasks increases. We propose a method that merges models trained on different tasks into a single model, maintaining strong performance across all tasks. Our approach leverages Jensen-Shannon divergence to guide the merging process without requiring additional labelled data, and automatically balances task importance. Unlike existing methods, our approach remains robust as the number of tasks grows and consistently outperforms prior work.
\end{abstract}

\section{Introduction}
\label{sec:introduction}

Current transformer based language models~\cite{vaswani2017attention,brown2020language} have demonstrated remarkable efficiency in handling a wide range of tasks within a single unified architecture. This led to the creation of the so-called Instruct models~\cite{shengyu2023instruction}, which are now state-of-the-art in almost all NLP tasks.
However, creating these models is very expensive and relies on large data collections and models.
With this goal of reducing costs in mind, a paradigm has been revived: {\bf model merging}. This paradigm rooted in ensemble methods consists in combining parameters of models specialized on specific tasks in order to create a new one that has new properties such as multi-tasking --- to only cite this one. 
This paradigm is all the more interesting given the large number of specialized models for specific tasks available on collaborative platforms such as {\tt HuggingFace}~\cite{wolf2019huggingface}.
Thus, since the work of~\citet{ilharco2022editing} a whole series of studies on model merging have been published, proposing many different methods to combine model parameters in order to produce a multi-tasking type model~\cite{goddard2024arcee,yang2024modelmergingllmsmllms}.
The motivation behind the creation of these different methods is to answer the following question: how can models trained on different tasks be combined without losing performance on each task? In the literature this problem is also mentioned as {\bf interference} between models. 
However, the wide variety of methods available and the lack of {\it consensus} in the production of methods seem to suggest that this question remains unanswered.
In this paper, we propose a new model merging method that allows us to create a so-called multi-task model. Our method formally addresses the interference problem mentioned above, which can adversely affect the model resulting from the combination. More formally, the contributions of this paper are the following three points:
\begin{itemize}[leftmargin=*]
    \item[1.] \textbf{A novel merging method.} We propose a merging method that is grounded in information theory~\cite{Cover91}. This method automatically learns how to combine parameters of different models in a data driven but reference free setting to achieve the multi-task property. We formally demonstrate that our method is linked to classical multi-task learning~\cite{caruanaMultitaskLearning1997} and is constructed in a way that it minimizes interferences between models by respecting weight disentanglement defined by~\citet{ortiz2023task}.
    \vspace{-0.25cm}
    \item[2.] \textbf{Better performances.} On a classical set-up with only two tasks to merge we illustrate that our method is the best one in average among classical state of the art methods. 
    \vspace{-0.25cm}
    \item[3.] \textbf{Better scalability.} We empirically show that our method scales better when the number of tasks to merge increases. This clearly illustrates that our method is more effective at limiting the possible effects of interference between the different models we wish to combine.
\end{itemize}

\section{Related Work}
\label{sec:rel-work}

\paragraph{Multi-Task Learning.}
Multi-task learning~\cite{caruanaMultitaskLearning1997} refers to methods that produce a model capable of solving several tasks. Classical multi-task learning generally consists in combining datasets from multiple tasks and training a model on this union, which is justified by classical results such as Stein's paradox~\cite{stein1956inadmissibility}. This philosophy is at the core of the production of most current models such as T5~\cite{raffel2020exploring} and more recently Instruct type models~\cite{shengyu2023instruction}. 
Although the results of these models are now considered state-of-the-art in NLP, this approach is not without its issues. Indeed, several studies~\cite{baxterModelInductiveBias2000,ben-davidExploitingTaskRelatedness2003,fiftyEfficientlyIdentifyingTask2021,standley2020tasks,jeong2025selective,maurer2016benefit} show that the choice of tasks is important in order to limit effects such as interference~\cite{yu2020gradient}: some tasks may have a negative effect on others.
In this study we propose a novel method, grounded in Information Theory~\cite{Cover91}, to produce a multi-task model. Our method is theoretically connected to the classical multi-task learning, and model merging.

\paragraph{Model Merging.}
While the idea of combining several models has its roots in ensemble methods originally developed for variance reduction~\cite{10.1007/3-540-45014-9_1}, model merging refers to methods that combine models in the parameter space to produce a new model that has new properties that go beyond variance reduction. While most applications of model merging are described by~\citet{yang2024model}, we recall here the most popular ones.
One objective of model merging is of course variance reduction ({\it i.e.} better generalization)~\cite{jin2022dataless,matena2022merging,ferretprojeter,izmailov2018averaging,wortsman2022modelsoupsaveragingweights}, with the goal of producing a more reliable model on a single target task by merging models trained on the same task with different initializations or hyper-parameter settings.
The most popular objective in model merging is the multi-task one~\cite{ilharco2022editing,yang2024adamergingadaptivemodelmerging,yu2024languagemodelssupermario,pfeifferAdapterFusionNonDestructiveTask2021a}: by merging models specialized on different tasks, the resulting model should achieve good performance on all tasks. \citet{zhou2024metagpt} and~\cite{ortiz2023task} provide theoretical insights into why model merging can work well for this multi-task objective.
Another objective is task unlearning~\cite{ilharco2022editing,kuo2025exact,kim2024negmerge}: by merging certain models (via addition and negation), we seek to remove or "forget" specific components.
Finally, modular learning~\cite{ballard1987modular,pfeifferModularDeepLearning2023,chronopoulouLanguageTaskArithmetic2023} is an interesting but less explored objective, which consists of creating a model that performs well on a task by merging models trained on other (possibly unrelated) tasks, leveraging the notion of transfer between tasks.
In this study, we propose a new model merging method with the goal of multi-task learning.

\paragraph{Merging Methods.}
While the literature offers a wide range of merging methods, some of them stand out with interesting results and properties. Since this study is focused on multi-task learning, we give a quick overview of methods designed for this.
The most straightforward and simple approach is model averaging~\cite{wortsman2022modelsoupsaveragingweights}, also known as isotropic merging, which simply consists in taking the uniform average of the models' parameters.
In~\cite{ilharco2022editing} the notion of task vector is introduced which is the shift in the parameter space from a pre-trained model to a fine-tuned one\footnote{We provide a more formal definition in~\autoref{sec:formalism}}. This concept has led to the framework of task arithmetic (TA) which is now extensively used~\cite{ortiz2023task} and has proven its efficiency across a wide range of applications. 
To enhance TA, SLERP (Spherical Linear Interpolation)~\cite{jang2024sphericallinearinterpolationtextanchoring} proposes a merging method that preserves certain geometric properties of the task vectors, thereby helping to mitigate interference as described earlier.
This notion of interference is also at the heart of a wide range of methods that attempt to address this problem by following a two-step process: first, task vectors are modified using techniques such as masking or singular value decomposition (SVD)~\cite{wang2024localizing, yadav2023tiesmergingresolvinginterferencemerging, stewart-siam-93}; second, the preprocessed task vectors are interpolated to produce the merged model. This is the case of methods such as TIES~\cite{yadav2023tiesmergingresolvinginterferencemerging}, AdaMerging++~\cite{yang2024adamergingadaptivemodelmerging}, or DARE~\cite{yu2024languagemodelssupermario}.
Among these methods, AdaMerging stands out as it is data-driven, {\it i.e.} this method automatically learns the best way to combine each model’s parameters. It requires a learning algorithm that uses data from the different tasks.
In this study, we propose a novel multi-task merging method rooted in task arithmetic and, like AdaMerging, leverages reference-free data-driven optimization.

\section{Formalism}
\label{sec:formalism}

\paragraph{Notations.} Random variables are denoted by capital letters ({\it e.g.}, $X$), their spaces by calligraphic letters ({\it e.g.}, $\mathcal{X}$), and elements by lowercase letters ({\it e.g.}, $x \in \mathcal{X}$). $\mathcal{P}(\mathcal{X})$ is the set of probability measures on $\mathcal{X}$, and $\mathcal{P}(\mathcal{Y}|\mathcal{X})$ the set of conditional probabilities on $\mathcal{Y}$ given $\mathcal{X}$. For $X \in \mathcal{X}$, $\mathbb{P}_X \in \mathcal{P}(\mathcal{X})$ is its law, and $\mathcal{S}_X \subset \mathcal{X}$ its support.
A task $t$ is a probability measure $\mathbb{P}_{X_t, Y_t} \in \mathcal{P}(\mathcal{X} \times \mathcal{Y})$, following the formalism of~\cite{fosse2025statistical}. For sake of simplicity, we hypothesis that all tasks share the same space $\mathcal{X} \times \mathcal{Y}$, and that $\mathcal{X}=\mathcal{Y}$, which is true in most generative tasks: both inputs and outputs are texts.
We note a language model with parameters $\theta$ with $M(\cdot|\cdot\,;\theta) \in \mathcal{P}(\mathcal{Y}|\mathcal{X})$. For task $t$, the specialized model on $t$ is $M(\cdot|\cdot\,;\theta_t)$ or $M_t(\cdot|\cdot)$.

\begin{rem}\label{rem:lm-prob}
    We identify a language model with a conditional probability: $M(\cdot|x;\theta)$ is a probability distribution over texts. For sake of simplicity we will refer to a language model by $M(x;\theta)$ when the model is given the data $x$ as input. In other words a language model is a communication channel.
\end{rem}

\subsection{Task vectors and model merging}

Task vectors have become essential objects in modern machine learning, given the vast number of fine-tuned models available on collaborative platforms such as HuggingFace~\cite{wolf2019huggingface}. These objects were first defined in~\cite{ilharco2022editing} as follows: if we denote by $\theta_0 \in \mathbb{R}^d$ the parameters of a pre-trained model ($d$ being the number of parameters in the model), and by $\theta_t \in \mathbb{R}^d$ the parameters of the same model after fine-tuning on a task $t$, the task vector is given by,
\begin{equation}
    \tau_t \triangleq \theta_t - \theta_0 \in \mathbb{R}^{d}.
    \label{eq:task-vector}
\end{equation}
The main approach in model merging is to combine such task vectors to create a new one that has new properties. In the following, we denote any task vector-based merging algorithm as,
\begin{equation}
    f\left(\theta_0, \set{\tau_t}, \Gamma \right) \in \mathbb{R}^d,
    \label{eq:mm}
\end{equation}
where $\theta_0$ are the pre-trained model parameters, $\set{\tau_t}$ is the set of task vectors to be merged, and $\Gamma$ is the set of parameters for the merging method. For example, task arithmetic~\cite{ilharco2022editing,ortiz2023task} (TA) can be formulated as,
\begin{equation}
    f\left(\theta_0, \set{\tau_t}, \Gamma \right) = \theta_0 + \sum_t \Gamma_t \times \tau_t,
    \label{eq:ta}
\end{equation}
where $\Gamma_t$ are real values (possibly negative). In the following we will refer to the TA method as $\Phi_n^\Gamma$, with $n$ the number of task vectors, and $\Gamma$ the merging coefficients. Estimating the parameters $\Gamma$ can be challenging, computationally intensive, and it strongly depends on the objective we aim to achieve with model merging ({\it e.g.}, multi-task learning, modularity, task unlearning, {\it etc.}). In this study, we focus on the {\bf multi-task setup} in model merging: the merged model must be capable of solving all the tasks on which the individual component models have been fine-tuned. In this set-up, as stated in~\cite{ortiz2023task}, one of the main properties we hope to achieve is {\bf weight disentanglement}, which is defined in~\autoref{def:weigth-dis}.

\begin{defn}[Weight Disentanglement~\cite{ortiz2023task}]\label{def:weigth-dis}
    Let $M(.;\theta)$ be a model parametrized by $\theta$. 
    Consider a set of tasks, $\set{(X_t, Y_t),~t \in \Trond}$, and their corresponding task vectors $\set{\tau_t,~ t\in\Trond}$ relatively to the model $M$. If tasks have non overlapping supports {\it i.e.} $\Srond_{X_t} \cap \Srond_{X_{t^\prime}} = \emptyset$ for all $t\neq t^{\prime}$, we say that a merging method $f$ satisfies weight disentanglement iff,
    \[
    \resizebox{\columnwidth}{!}{$
    M\left(x; f\left(\theta_0, \{\tau_t\}, \Gamma \right) \right) = 
    \begin{cases}
    M(x; \theta_0 + \tau_t) & \text{if } x \in \Srond_{X_t}, \\
    M(x; \theta_0) & \text{if } x \notin \bigcup_{t=1}^T \Srond_{X_t}.
    \end{cases}
    $}
\]
\end{defn}

In other words, a merging method satisfies~\autoref{def:weigth-dis} if adding $\tau_t$ does not affect model's output outside the corresponding task support $\Srond_{X_t}$. A merging method satisfying weight disentanglement assures that the performance on all the merged tasks will be preserved, assuring thus good behaviour in terms of multi-tasking. 
In this study we will focus on TA merging algorithm denoted as $\Phi_n^\Gamma$. We choose this method because, as explained in~\autoref{app:everything-is-ta}, most current model merging approaches are based on task arithmetic. However, as we will demonstrate, our theoretical framework is not limited to task arithmetic, and most of our results can be extended to more general methods.

\subsection{Our Method}

Based on the Kullback and Leibler ($\mathrm{KL}$)~\cite{kullback1951information} and Jensen-Shannon ($\mathrm{JS}$)~\cite{wong1985entropy} divergences, we propose a method to automatically estimate the merging coefficients in~\autoref{eq:ta} in a way that forces weight disentanglement ({\it c.f.}~\autoref{def:weigth-dis}), which translates into preserving performance (as much as possible) on the different task components. 
First, we recall some basic properties of the $\mathrm{KL}$ divergence. Given two discrete probability distributions $\mu$ and $\nu$ over some discrete space $\Irond$ we have,
\[
    \mathrm{KL}(\mu \| \nu) \triangleq \sum_{i \in \Irond} \mu(i) \log\left(\frac{\mu(i)}{\nu(i)}\right).
\]
In this definition, we refer to $\mu$ as the \textbf{reference} distribution. From this, we also define the $\mathrm{JS}$ divergence, which is a symmetric version of the $\mathrm{KL}$ divergence,
\[
\resizebox{\columnwidth}{!}{$
    \mathrm{JS}(\mu, \nu) \triangleq \frac{1}{2} \left(
        \mathrm{KL}\left(\mu \,\middle\|\, \frac{\mu + \nu}{2}\right) 
        + \mathrm{KL}\left(\nu \,\middle\|\, \frac{\mu + \nu}{2}\right)
    \right).
$}
\]
Both divergences are positive real numbers that quantify a notion of distance between the measures $\mu$ and $\nu$ (lower values indicate closer distributions).
We can extend these definitions to define divergence between language models.
Given two LMs, $M_1$ and $M_2$, and a reference input dataset $X$, we define the divergence $\mathrm{D}$ (either $\mathrm{KL}$ or $\mathrm{JS}$) between these models as,
\[
    \mathrm{D}_X(M_1 \| M_2) \triangleq \mathbb{E}_{X}\left[\mathrm{D}\big(M_1(\cdot \mid X) \| M_2(\cdot \mid X)\big)\right].
\]
This is one possible expression for the divergence between transition probabilities. Since, in NLP tasks, the model generates sequences of text, we can naturally extend the formula above to sequence-level distributions. See~\autoref{app:divergence_sequence_case} for more details. 
Since this work focuses on task vectors and model parameters, in the following, we will denote a language model (LM) by its parameters, {\it i.e.} $\theta_t \equiv M_t$. Given a divergence $\mathrm{D}$ (either $\mathrm{KL}$ or $\mathrm{JS}$), we provide in~\autoref{eq:meth-obj} an optimization problem to automatically find coefficients in task arithmetic.

\begin{equation} \label{eq:meth-obj}
    \Gamma^* \triangleq \underset{\Gamma}{\arg\min}~ \sum_{t=1}^n \mathrm{D}_{X_t}\left(\theta_t \,\|\, \Phi_n^\Gamma\right).
\end{equation}
    
\begin{rem} \label{rem:centroid}
    \cite{nielsen2020generalization} provides intuition for the solution of~\autoref{eq:meth-obj}. In the case where $\mathrm{D} = \mathrm{JS}$, the solution corresponds to the probabilistic centroid of the different output distributions generated by the models $\{\theta_t\}$, which is a concept firstly introduced in the framework of Information Geometry~\cite{amari2000methods}. In~\autoref{app:everything-is-ta} we give more details about properties of some centroids in the case of model merging.
\end{rem}

\begin{rem}\label{rem:hp-est}
    \autoref{eq:meth-obj} is defined for the task arithmetic merging function. However, we can replace $\Phi_n^\Gamma$ with any merging method $f(\theta_0, \{\tau_t\}, \Gamma)$ (not only task arithmetic).
\end{rem}

Despite the simplicity of the optimisation problem in~\autoref{eq:meth-obj}, we propose several results that connect with weight disentanglement and multi-task learning theory.

\begin{prop}\label{prop:0-loss}
    For either $D_X = \mathrm{KL}$ or $D_X = \mathrm{JS}$, the objective function defined in~\autoref{eq:meth-obj} has a minimum value of $0$ if and only if the merging method satisfies weight disentanglement around $\theta_0$ on the merged tasks.
\end{prop}
    
The proof of this result is given in~\autoref{app:subsec:proof-0-loss}. Therefore, the optimization problem in~\autoref{eq:meth-obj} admits a minimum value of $0$ iff there exists a choice of merging parameters $\Gamma$ that respects the weight disentanglement property, giving added weight to this property.
Furthermore, the following proposition establishes an interesting link between~\autoref{eq:meth-obj} and classical multi-task learning setups ({\it i.e.}, when datasets are merged before fine-tuning), further emphasizing the relevance of the objective function defined in~\autoref{eq:meth-obj}.

\begin{prop}\label{prop:meth-eq-mt}
    For $\mathrm{D} = \mathrm{KL}$, the optimization problem defined in~\autoref{eq:meth-obj} is a Moment projection ($\mathrm{M}$-projection) approximation of the multi-task objective when tasks are merged before training.
\end{prop}

We give the proof of~\autoref{prop:meth-eq-mt} in~\autoref{app:subsec:mt-connex}. This proposition is noteworthy as it links our method to the classical multi-task learning objective considered so far as the strongest baseline. In other words our method is an approximation of the classical multi-task learning objective. Moreover, the demonstration of~\autoref{prop:meth-eq-mt} indicates that the effectiveness of this approximation depends on the performance of the individual fine-tuned models.
Finally, we give an illustration of our procedure for merging two models in~\autoref{fig:kl-merging}, and we provide the general algorithm for our approach in~\autoref{algo:meth}.

\begin{algorithm}[tb]
    {\small
    
    \caption{\small Model Merging via Divergence-Based Optimization. In this procedure, for a sequence $y$ and a model $M$, $\mathrm{Logits}(y, M)$ denotes the logits (soft probs) of the sequence $y$ given by the model $M$. From a technical point of view, this is only a forward pass of $y$ through $M$. In~\autoref{app:divergence_sequence_case} we propose more details about the computation of divergences.}
    \label{algo:meth}
    
    \begin{algorithmic}
        \Require : \\
        $\{X_t \mid t \in \mathcal{T}\}$ \\
        $f_\Gamma$ (Merging method) \\
        $\theta_0$ (Pre-trained model) \\
        $\{\theta_t \mid t \in \mathcal{T}\}$ (Fine-tuned models) \\
        \textcolor{charteGreen}{// Get logits of the data by generation}
        \For{$t \in \Trond$}
            \For{$x \sim X_t$}
                \State $\hat{y}_t^x \gets \{y^1, \ldots, y^{m}, \text{eos}\} \sim M(x;\theta_t)$
                \State $\ell(t, x) \gets \mathrm{Logits}(\hat{y}_t^x, M(x, \theta_t))$
            \EndFor
        \EndFor \\
        
        \textcolor{charteGreen}{// Compute coefficients}
        \State $\Gamma_i \gets \frac{1}{|\Trond|}\quad \forall i$ (Init. of coefficients)
        \For{each epoch}
            \For{$b\sim \cup_t X_t$} \textcolor{charteGreen}{// Batch sampling}
                \For{$x \in b$}
                    \State $\ell(f, x) \gets \mathrm{Logits}(\hat{y}_t^x, M(x, f_\Gamma))$
                \EndFor
                \State $L_\Gamma \gets 0$
                \For{each $t \in \mathcal{T}$}
                    \For{$x \in b$} 
                        \State \textcolor{charteGreen}{// Choose right task}
                        \State \textcolor{charteGreen}{// Compute pointwise Loss}
                        \State $L_\Gamma \gets L_\Gamma + \mathrm{D}(\ell(t,x)\| \ell(f,x))$
                    \EndFor
                \EndFor
                \State $\Gamma \gets \Gamma - \nabla_\Gamma L_\Gamma$ \textcolor{charteGreen}{// Gradient update}
            \EndFor
        \EndFor
        \State \Return $\Gamma$
    \end{algorithmic}
    }
\end{algorithm}

\begin{figure*}[tb]
    \centering
    \includegraphics[width=1\textwidth]{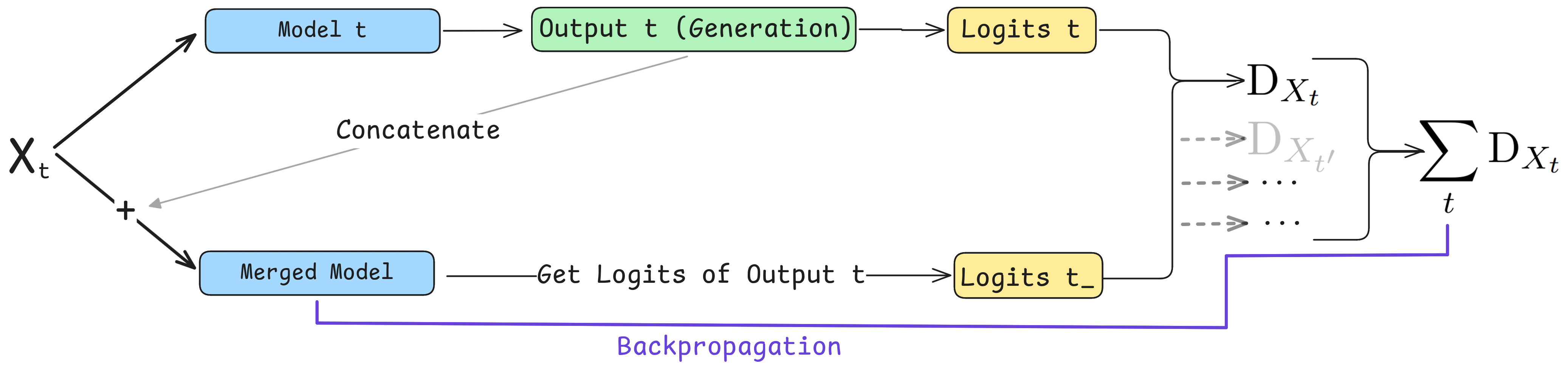}
    \caption{Illustration of the divergence-based model merging method. This figure shows the merging loss associated to one task to be merged, which is the task $t$. Our method consists in doing this procedure for every task $t$ and, as it is written on the left side, to sum all the associated loss. For more details, we refer to~\autoref{algo:meth}.}
    \label{fig:kl-merging}
\end{figure*}

\section{Experimental Protocol}
\label{sec:expe-proto}

\paragraph{Models and datasets.} 
As stated in~\autoref{sec:formalism}, our proposed merging method is designed for generative language models. Therefore, we evaluate our approach on two architectures: a decoder-only architecture using the \href{https://huggingface.co/Qwen/Qwen2.5-0.5B}{\tt Qwen2.5-0.5B}~\cite{qwen2.5} model and an encoder decoder architecture using the \href{https://huggingface.co/google-t5/t5-base}{\tt T5-Base}~\cite{raffel2020exploring} model. We apply our method on both classification and generative tasks.
For classification tasks, we used the GLUE benchmark~\cite{wang2019gluemultitaskbenchmarkanalysis}, as is common in recent studies. We fine-tuned {\tt Qwen2.5-0.5B} on 7 classification tasks, resulting in 7 distinct checkpoints for our experiments. Fine-tuning was performed using standard supervised fine-tuning (SFT), where the model generates the classification labels. An example of the fine-tuning data is shown in~\autoref{tab:classif-form}.
For generative tasks, we used the {\tt T5-Base} model and existing checkpoints on several tasks, including IMDB~\cite{imdb-data}, QASC~\cite{qasc-dataset}, SQuAD~\cite{squad-dataset}, and CommonGen~\cite{commongen-dataset}\footnote{All checkpoints are available \href{https://huggingface.co/mrm8488}{here}}. This setup allows us to evaluate our method in a more realistic scenario, where we do not control the fine-tuning process and only have access to the fine-tuned models. We demonstrate that our method is effective in this setting as well.

\begin{table}[tb]
    \centering
    {\small \begin{tabular}{l l}
        \toprule
        {\bf Description} & {\bf Data} \\
        \midrule
        \multirow{2}{*}{Prompt} & Is the following sentence linguistically \\
        & acceptable or unacceptable in english?\\
        \hline
        \multirow{2}{*}{Data} & Sodium is a little too peppy for me to want \\
        & to try mixing and water in a teacup. \\
        \hline
        \multirow{2}{*}{Labels} & Answer with acceptable or unacceptable.  \\
        & Answer: \\
        \hline
        \multirow{1}{*}{Answer} & unacceptable<|endoftext|> \\
        \bottomrule
    \end{tabular}}
    \caption{Data example for the fine-tuning on the GLUE Benchmark. This example is derived from the CoLA dataset. Fine-tuning is done in a completion only fashion (SFT) after the pattern {\it Answer:}}
    \label{tab:classif-form}
\end{table}

\paragraph{Evaluation Metrics.}
To assess the quality of a merging method, we define a metric to quantify how the merged model performs compared to each fine-tuned model. This metric is the Average Normalized Performance (ANP), defined as:
\begin{equation}
    \mathrm{ANP} \triangleq \frac{1}{n} \sum_{t=1}^{n} \frac{\mathrm{PERF}\left( f\left(\theta_0, \{\tau_i\}, \Gamma \right); t \right)}{\mathrm{PERF}(\theta_t; t)},
    \label{eq:perf}
\end{equation}
where $n$ is the number of merged tasks, $\mathrm{PERF}\left( f\left(\theta_0, \{\tau_i\}, \Gamma \right); t \right)$ is the performance of the merged model on task $t$ using the merging method $f$, and $\mathrm{PERF}(\theta_t; t)$ is the performance of the fine-tuned model on task $t$. In our work, $\mathrm{PERF}$ is measured by classical accuracy for classification tasks, and by the $\mathrm{ROUGE}_1$ score~\cite{lin-2004-rouge} for generation tasks.
Each performance metric is computed on a separate test set.
This metric quantifies the performance of a specific merging experiment. In practice, multiple merging experiments can be conducted (for example, by sampling different sets of $k$ tasks to merge). In such cases, we compute the $\mathrm{ANP}$ metric for each merging experiment and report the average. It is worth noting that the $\mathrm{ANP}$ metric is well suited for the multitasking setup, as it automatically provides a value relative to the baselines.\footnote{If the merging method verifies~\autoref{def:weigth-dis} $\Rightarrow$ $\mathrm{ANP} = 1$.}

\paragraph{Merging method.}
As stated previously, in this study we focus on the task arithmetic merging function $\Phi_n^\Gamma$. We define two versions of this merging method. The first, described in~\autoref{eq:ta}, is what we call {\bf Task Level} (TL) since one coefficient $\Gamma_i$ is given for each task. The granularity of the method can be further refined: since the models we use are deep neural networks organized in layers, we can define a different parameter $\Gamma_l^i$ for each layer for task $l$, resulting in $n \times L$ merging coefficients (where $L$ is the number of layers). We refer to this method as the {\bf Layer Level} (LL) approach. In the following, we present results for both methods and always specify which one is used.
Moreover, our method requires input data for each task; that is, when merging tasks $i$ and $j$, we need data from $X_i$ and $X_j$. For this purpose, we use the validation set of each task.

\paragraph{Baselines.}
Since we propose a new merging method, we compare our method with existing ones that are widely used such as model averaging~\cite{wortsman2022modelsoupsaveragingweights}, Multi-SLERP~\cite{goddard2024arcee}, TIES~\cite{yadav2023tiesmergingresolvinginterferencemerging}, and AdaMerging~\cite{yang2024adamergingadaptivemodelmerging} in its two variants: Task Level and Layer Level. For TIES, we followed the recommended recipe from~\cite{yadav2023tiesmergingresolvinginterferencemerging}. For Multi-SLERP, the weights associated to each tasks were set to $\frac{1}{n}$. For optimization-based methods (AdaMerging and ours), for each method we used the same hyper-parameters across all merging experiments, with a batch size of $4 \times n$ for each iteration. In~\autoref{tab:training-params} we provide more training details.

\section{Results}
\label{sec:results}

\subsection{Divergence justification} \label{sec:results:divergence-justification}

As stated in~\autoref{sec:formalism}, our method is based on computing divergences between the merged model and the various fine-tuned ones. Before applying this method, we decided to illustrate an interesting result: divergence and performance are well correlated.
Given a set of tasks $\Trond$, for each pair of tasks $(i,j)\in\Trond\times\Trond$, we compute $\mathrm{D}_{X_i}(\theta_i \| \theta_j)$ on a development set, and the performance of the model $\theta_j$ on task $i$ on a test set, denoted as $\mathrm{PERF}(\theta_j, i)$. Then, for all tasks $i$ we compute the correlation between $\set{-\mathrm{D}_{X_i}(\theta_i \| \theta_j),~\forall j\in\Trond}$ and $\set{\mathrm{PERF}(\theta_j, i),~\forall j\in\Trond}$. We report Spearman's correlations in~\autoref{tab:spearman_results} for classification tasks. We clearly observe that we obtain high correlations indicating that $\mathrm{KL}$ and $\mathrm{JS}$ are interesting proxies for performance: if $\mathrm{D}_{X_i}(\theta_i \| \theta_j)$ is low, it may suggest that $\mathrm{PERF}(\theta_j, i)$ will be high (without stating that this is a causation relation). We also observe that the $\mathrm{JS}$ divergence achieves the highest correlation; therefore, unless stated otherwise, we use the $\mathrm{JS}$ divergence criterion in our experiments. Although the $\mathrm{JS}$ divergence consistently outperforms $\mathrm{KL}$ divergence, the difference is not always significant.
Since our method involves minimizing $\sum_{i=1}^n\mathrm{D}_{X_i}(\theta_i \| \Phi_n^\Gamma)$, it thus can be viewed as aiming for a merging model that performs well on each task. This aligns directly with~\autoref{prop:meth-eq-mt}, which states that our method is an approximation of the classical multi-task learning objective. 
For more details, we refer to~\autoref{divergence_justification}.

\begin{table}[tb]
    \small
    \centering
    \begin{tabular}{lcc}
    \toprule
    $Task$ ($\theta_i$)  & $\mathrm{KL}$ & $\mathrm{JS}$ \\
    \midrule
     \textbf{CoLA}     & 0.812 & 0.925 \\
     \textbf{SST-2}    & 0.920 & 0.887 \\
     \textbf{QQP}      & 0.313 & 0.340 \\
     \textbf{QNLI}     & 0.774 & 0.796 \\
     \textbf{MNLI}     & 0.947 & 0.990 \\
     \textbf{RTE}      & 0.769 & 0.915 \\
     \textbf{MRPC}     & 0.877 & 0.875 \\
     \midrule
    \textbf{Avg.}  & \underline{0.773} & \textbf{0.818} \\
    \bottomrule
    \end{tabular}
    \caption{Spearman's correlation between $\set{-\mathrm{D}_{X_i}(\theta_i \| \theta_j),~\forall j\in\Trond}$ and $\set{\mathrm{PERF}(\theta_j, i),~\forall j\in\Trond}$ for all $i\in\Trond $ (The negation sign is added to have positive correlations). The "Avg." row reports the mean correlation across all tasks.}
\label{tab:spearman_results}
\end{table}

\subsection{Effective Model Merging}

We first illustrate our method in a pairwise model merging setup {\it i.e.} we merge only two tasks at a time and compute $\mathrm{ANP}$ as defined in~\autoref{eq:perf}. We computed this metric for every possible pairwise merging experiment and for each task types (classification and generation). For example, for classification tasks, we have 7 distinct tasks, meaning that we can perform $\binom{7}{2} = 21$ pairwise merging experiments and thus compute 21 distinct values of $\mathrm{ANP}$ (for generation tasks we have $\binom{4}{2} = 6$). \autoref{tab:merged_methods_avg} presents the average $\mathrm{ANP}$ metric across all pairwise merging experiment. 
We can clearly observe that our method achieves the best average performance.
Moreover we can see that the Layer Level variant outperforms the Task Level variant, which is expected since the former has a greater number of merging coefficients (one for each layer of each task), giving a higher degree of granularity.

\begin{table}[tb]
    \centering
        {\small \begin{tabular}{l c c}
        \toprule
        \textbf{Merging Method} & \textbf{Classif.} & \textbf{Gen.} \\
        \midrule
        Model Averaging        & 88.48 ($\pm$ 3.17) & 94.38 ($\pm$ 2.6) \\
        Multi-SLERP            & 91.54 ($\pm$ 2.98) & 76.39 ($\pm$ 21.04) \\
        TIES                   & 94.06 ($\pm$ 1.81) & 95.53 ($\pm$ 4.44) \\
        TL Adamerging  & 93.62 ($\pm$ 2.53) & 93.42 ($\pm$ 10.08) \\
        LL Adamerging & 94.06 ($\pm$ 2.95) & 83.20 ($\pm$ 9.94) \\
        \hline \hline
        TL $\mathrm{KL}$ (\textit{ours}) & 97.68 ($\pm$ 1.94) & 93.97 ($\pm$ 3.46) \\
        LL $\mathrm{KL}$ (\textit{ours}) & 99.16 ($\pm$ 0.50) & 97.50 ($\pm$ 1.73) \\
        \hline
        TL $\mathrm{JS}$ (\textit{ours}) & \underline{97.73} ($\pm$ 2.01) & \underline{97.29} ($\pm$ 1.94) \\
        LL $\mathrm{JS}$ (\textit{ours}) & \textbf{99.18} ($\pm$ 0.51) & \textbf{98.93} ($\pm$ 1.05) \\
        \bottomrule
        \end{tabular}}
    \caption{Average $\mathrm{ANP}$ (\%) for various merging methods across GLUE and T5 task pairs. The best results per benchmark are boldfaced, and the second best is underlined. See~\autoref{tab:merged_tasks_glue} and~\autoref{tab:merged_tasks_T5} in the appendix for more detailed results.}
\label{tab:merged_methods_avg}
\end{table}

\subsection{Robustness to Number of Tasks}

A major limitation of existing merging methods, as noted in~\cite{yadav2023tiesmergingresolvinginterferencemerging}, is their lack of robustness as the number of merged models increases.
We thus decided to empirically assess whether our method is robust to an increasing number of merged tasks by varying this number in our experiments and testing all possible combinations of tasks.
For example, in the classification setup, we have 7 different tasks. We test our merging method by merging between 2 and 7 tasks. Moreover, to ensure the reliability of our conclusions, we follow the same procedure as before: for each number of merged tasks, we perform all possible merging combinations. For instance, when merging three tasks, we consider all possible combinations, \textit{i.e.}, $\binom{7}{3} = 35$ merging experiments. We then compute the average $\mathrm{ANP}$ metric. In this analysis, we focus on the use of the $\mathrm{JS}$ divergence, as it demonstrated slightly better results.
\autoref{fig:multi-tasks} shows the average $\mathrm{ANP}$ metric across all possible merging experiments as a function of the number of tasks, for both the classification setup (\autoref{subfig:classif}) and the generation setup (\autoref{subfig:generation}).
First, we observe that regardless of the used method, increasing the number of tasks generally leads to a degradation in the average $\mathrm{ANP}$, which is consistent with the fact that task interference becomes more apparent.
We can also observe that, regardless of the number of merged tasks, our method (both task-level and layer-level) consistently provides better results, with curves that remain higher throughout the graph.
Moreover, the drop in performance as the number of tasks increases is less pronounced for our method, illustrating its robustness with respect to the number of tasks.

\begin{figure}[tb]
    \centering
    \begin{subfigure}{0.5\textwidth}
        \centering
        \includegraphics[width=\linewidth]{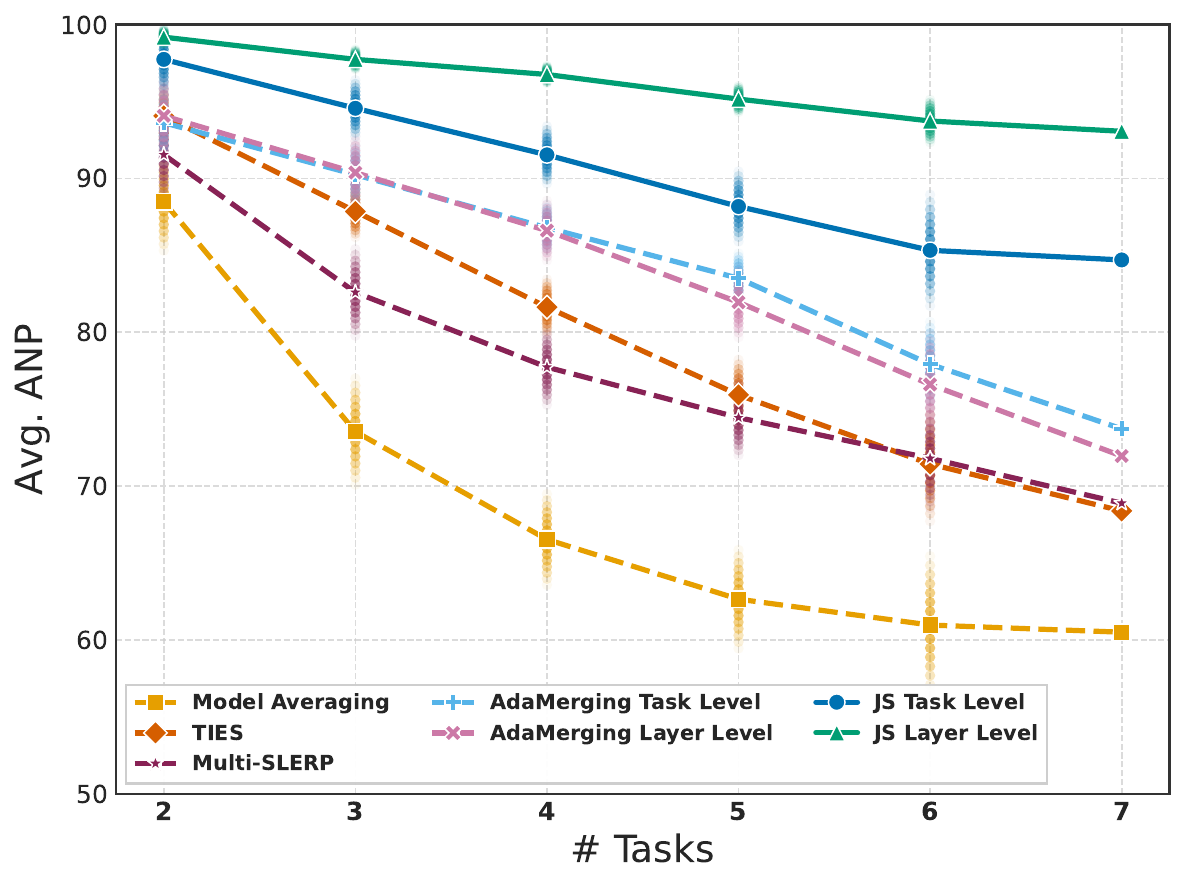}
        \caption{Classification \label{subfig:classif}}
    \end{subfigure}
    \begin{subfigure}{0.5\textwidth}
        \centering
        \includegraphics[width=\linewidth]{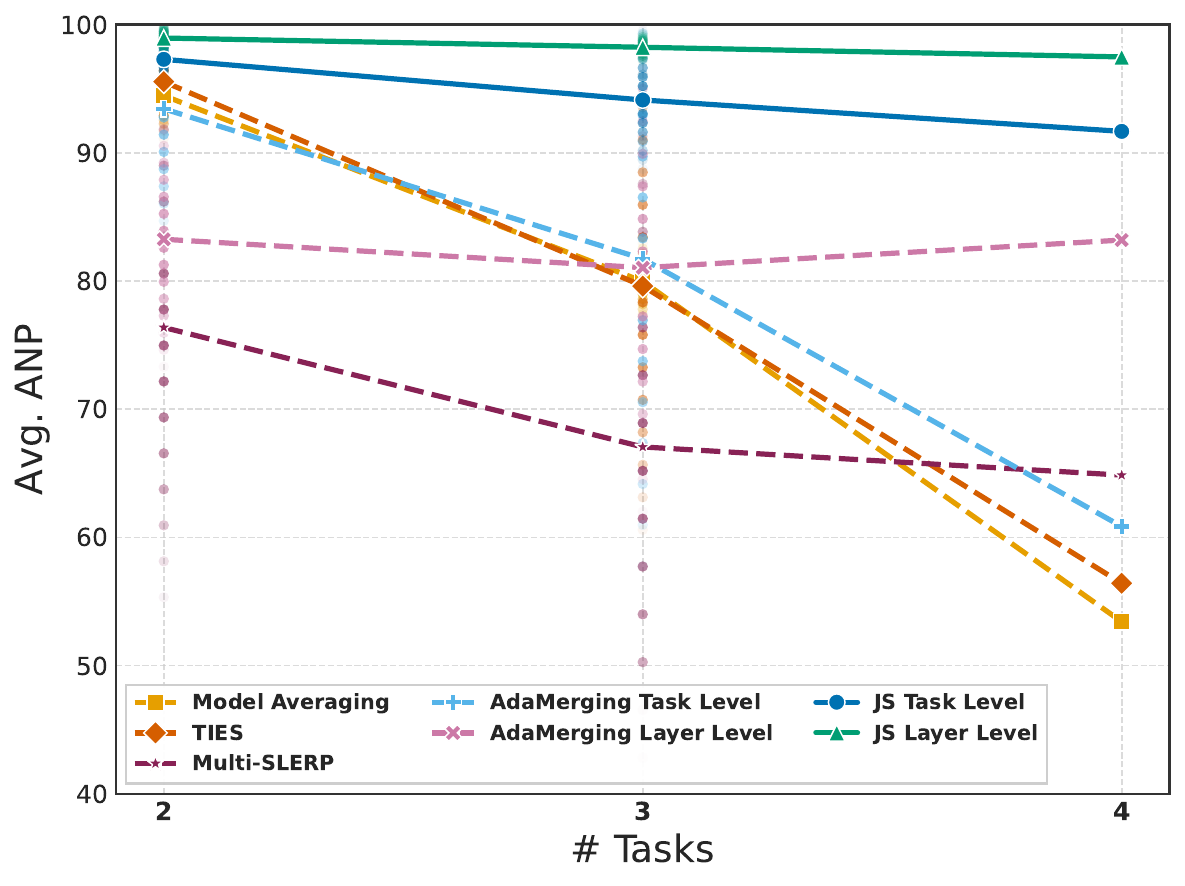}
        \caption{Generation \label{subfig:generation}}
    \end{subfigure}
    \caption{
        Evolution of the average $\mathrm{ANP}$ metric as a function of the number of merged tasks. For each number of tasks $k$ on the $x$-axis, several merging experiments were conducted ($\binom{n}{k}$ in total), and we report the 95\% confidence interval.
    }
    \label{fig:multi-tasks}
\end{figure}

In addition to the average $\mathrm{ANP}$, \autoref{fig:multi-tasks} presents confidence intervals (CI) across experiments for various tasks with a fixed number of tasks to merge. Our method demonstrates strong stability, both in classification and generation tasks. For generation tasks, some SOTA methods show large, overlapping confidence intervals, so we report CI margins in~\autoref{tab:merging_margins}. Notably, our method exhibits greater stability compared to Adamerging, another optimization-based approach.

\begin{table}[tb]
    \centering
    \small
    \renewcommand{\arraystretch}{1.2}
    \begin{tabular}{lcc}
    \toprule
    \textbf{Merging Method} & \textbf{2 Tasks} & \textbf{3 Tasks} \\
    \midrule
    Model Averaging          & $\pm$2.60 & $\pm$2.96 \\
    Multi-SLERP              & $\pm$21.04 & $\pm$27.97 \\
    TIES                     & $\pm$4.44 & $\pm$19.01 \\
    TL Adamerging    & $\pm$10.08 & $\pm$23.94 \\
    LL Adamerging   & $\pm$9.94 & $\pm$19.03 \\
    \hline \hline
    TL $\mathrm{JS}$ (ours)     & $\pm$1.94 & $\pm$5.37 \\
    LL $\mathrm{JS}$ (ours)    & $\pm$1.05 & $\pm$1.09 \\
    \bottomrule
    \end{tabular}%
    \caption{Confidence interval (CI) margins for different merging methods when merging 2 or 3 tasks. For each method, the margin is measured across multiple merging experiments performed on different task combinations.}
    \label{tab:merging_margins}
\end{table}

\subsection{Method Behaviour Analysis}

In this section, we analyse our method in greater depth by examining its convergence behaviour. We focus here exclusively on classification tasks.

\paragraph{Performance Convergence.}
Since our method is data-driven, it requires a training procedure, which we described in~\autoref{algo:meth}. We decided to further investigate its training dynamics by studying the evolution of the $\mathrm{ANP}$ metric over different training iterations. Here, we focus on pairwise merging, and \autoref{fig:acc_iter} shows the evolution of the $\mathrm{ANP}$ metric across training iterations for both the Task Level and Layer Level variants.
First we can notice once again that Task level and Layer level provide similar results which is in line with our previous findings.
Then, in all cases, the merging process converges smoothly without signs of over-fitting ({\it i.e.}, a sudden drop in performance), indicating that the proposed methods are effective and stable, consistently merging task-specific representations over training iterations.

\begin{figure}[tb]
    \centering
    \includegraphics[width=1.\linewidth]{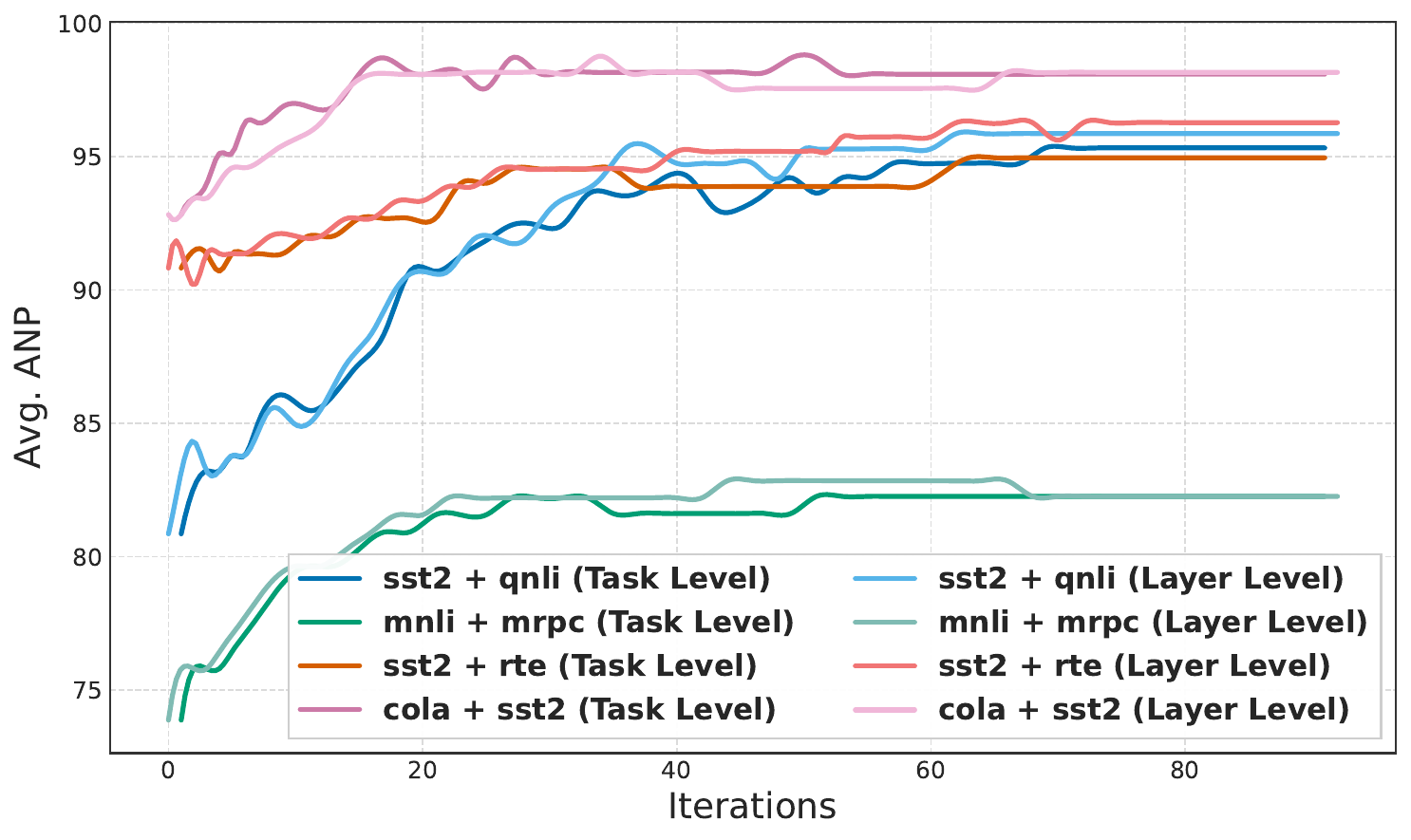}
    \caption{$\mathrm{ANP}$ metric of merged task pairs with Task Level and Layer Level JS Divergence as a function of training iterations.}
    \label{fig:acc_iter}
\end{figure}

\paragraph{Dataset Size Influence.}
As described in~\autoref{sec:formalism}, when merging a set of task vectors $\set{\tau_i}$, we require some data derived from $\set{X_i}$. A natural and important question is how much data is needed for our method to achieve strong performance. To investigate this, we studied the evolution of the $\mathrm{ANP}$ metric as a function of the amount of data used by our method. For this experiment, we focused on classification tasks and considered three merging scenarios: (CoLA, SST-2), (QNLI, MNLI), and (RTE, MRPC).
In~\autoref{fig:glue_dataset}, we plot the evolution of the $\mathrm{ANP}$ metric. We observe that our method outperforms state-of-the-art methods with as few as 25 samples, which corresponds to only 0.4\% of the training corpus used for fine-tuning the merged models and 5\% of the validation dataset.

\begin{figure}[tb]
    \centering
    \includegraphics[width=1.\linewidth]{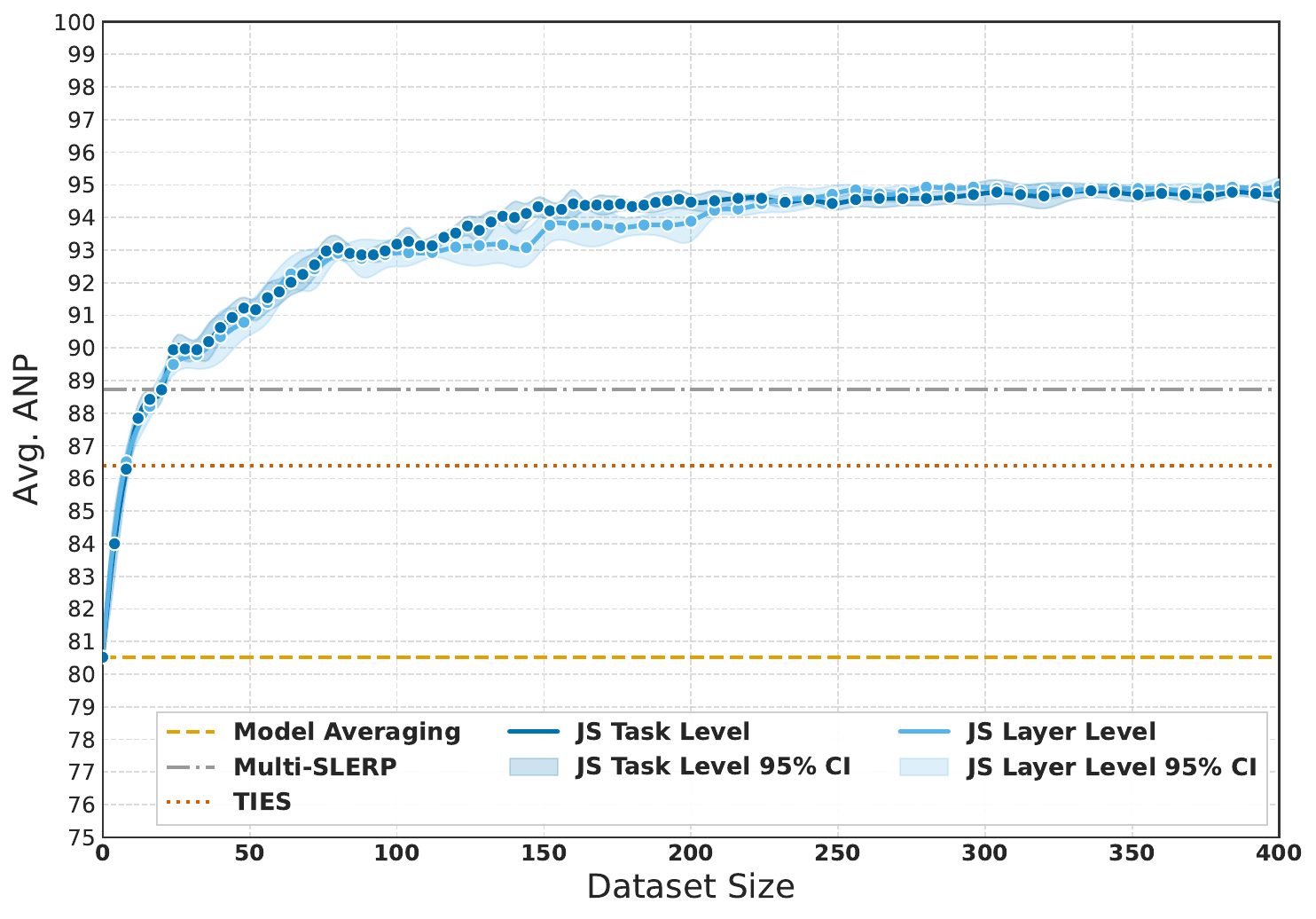}
    \caption{Impact of dataset size on the performance of our approach compared to data-free baselines. The average $\mathrm{ANP}$ is computed across three task pairs.}
    \label{fig:glue_dataset}
\end{figure}

\section{Conclusion}
\label{sec:conclusion}

In this work, we propose a new, data-driven but reference free, merging method which consists in finding the probabilistic centroid of fine-tuned models in order to produce a multi-task model.
After showing that, theoretically, our method is directly linked to multi-task learning and the concept of weight disentanglement, we demonstrate, empirically, that our method consistently outperforms most state-of-the-art methods on a pairwise model merging set-up.
Furthermore, we show that our method seems to better handle interference issues considering that it is the best one when the number of merged tasks increases.
Finally, we show that our method has high training stability and requires a relatively small amount of data to work.

\newpage

\section*{Limitations}
\label{sec:limitations}

Despite interesting results, our method present several limitations that we tempt to address here.

\paragraph{Other fine-tuning methods.} Our method has been extensively tested when the specialized models were constructed using full-finetuning. In this setup, the task vectors are sparse, and as a consequence, the interference problem is more limited. However, in low rank adaptation (LoRA)~\cite{hu2022lora} fine-tuning, task vectors ({\it i.e.}, LoRA matrices) affect the task arithmetic paradigm and are responsible for significant performance loss when merging. A limitation of our work is that we have not experimented within this constrained setup. 

\paragraph{Dataset influence.} Our method assumes that for each task $t$, we have access to a sample of the distribution $\Prob_{X_t}$ corresponding to the input data for task $t$. However, in some setups, we may not have access to such a distribution, but only to an approximation of it ($\Prob_{\tilde{X}_t}$). We have not addressed this case here. On the other hand, we propose an initial theoretical analysis of such a case in~\autoref{app:th-res:distribution-shift}. We believe that pushing in this direction will provide more robust results for model merging.

\newpage
\bibliography{custom,biblio}

\onecolumn
\newpage
\appendix

\section{Theoretical results}
\subsection{Proof of~\autoref{prop:0-loss}}\label{app:subsec:proof-0-loss}

\begin{proof}
    Since the divergences we use are non-negative, we have the following equivalence:
    \[
        \sum_{t=1}^n \mathrm{D}_{X_t}\left(\theta_t \,\|\, \Phi_n^\Gamma\right) = 0  
        \Leftrightarrow  
        \forall t,~\mathrm{D}_{X_t}\left(\theta_t \,\|\, \Phi_n^\Gamma\right) = 0.
    \]
    Moreover, by the properties of the $\mathrm{KL}$ and $\mathrm{JS}$ divergences, we have:
    \[
        \mathrm{D}_{X_t}\left(\theta_t \,\|\, \Phi_n^\Gamma\right) = 0 
        \Leftrightarrow 
        \forall x \in \mathcal{S}_{X_t},~M(x; \theta_t) = M(x; \Phi_n^\Gamma),
    \]
    where the last equality is understood in the sense of equality of measures. By transitivity, we obtain:
    \[
        \sum_{t=1}^n \mathrm{D}_{X_t}\left(\theta_t \,\|\, \Phi_n^\Gamma\right) = 0 
        \Leftrightarrow 
        \forall t,~\forall x \in \mathcal{S}_{X_t},~M(x; \theta_t) = M(x; \Phi_n^\Gamma),
    \]
    which concludes the proof.
\end{proof}

\begin{rem}
    In this demonstration, we stated that this was due thanks to some properties of the $\mathrm{KL}$ or $\mathrm{JS}$ divergence. However, we have the same result if we use any $f$-divergence, any divergence than can be expressed as following,
    \[
        \mathrm{D}_f(\mu \| \nu) = \int f\left(\frac{d\mu}{d\nu}\right) d\nu,
    \]
    which is of course the case of the Jensen Shannon and the Kullback ones. In fact this proof is valid for any divergence $\mathrm{D}$ which satisfies the following property,
    \[
        \mathrm{D}_f(\mu \| \nu) = 0 ~\Leftrightarrow~ \mu = \nu
    \]
\end{rem}

\subsection{Proof of~\autoref{prop:meth-eq-mt}}
\label{app:subsec:mt-connex}

\begin{defn}[Multi task objective]
    Let $\{(X_t, Y_t) \mid t \in \mathcal{T}\}$ be a set of tasks, $\mathcal{H}$ the cross-entropy loss function, and $M(\cdot; \theta)$ a model parameterized by $\theta$. We define the multi-task loss function as follows:
    \[
        \Lrond_{\textrm{MT}}(\theta) \triangleq \frac{1}{|\Trond|} \sum_{t \in \Trond} \Hrond \left( \Prob_{Y_t | X_t}, M(X_t; \theta) \right),
    \]
\end{defn}

\begin{lem}\label{lemma:cross-entropy}
    Let $(X_t, Y_t)$ be a task, $\mathcal{H}$ the cross-entropy loss function, and $M(\cdot;\theta)$ a model parameterized by $\theta$. Then, the following relation holds:
    \[
        \mathcal{H}(\mathbb{P}_{Y_t | X_t}, M(X_t; \theta)) = H(Y_t | X_t) + \mathrm{KL}(\mathbb{P}_{Y_t|X_t} \| M(X_t;\theta)),
    \]
    where $H$ denotes Shannon's entropy.
\end{lem}

We now provide the proof of~\autoref{prop:meth-eq-mt}:

\begin{proof}
    By hypothesis, we have
    \[
        \theta_t = \underset{\theta}{\arg\min}~\mathcal{H}(\mathbb{P}_{Y_t|X_t}, M(X_t; \theta)).
    \]
    By~\autoref{lemma:cross-entropy}, this is equivalent to
    \[
        \theta_t = \underset{\theta}{\arg\min}~\mathrm{KL}(\mathbb{P}_{Y_t|X_t} \| M(X_t;\theta)).
    \]
    Thus, $M(X_t; \theta_t)$ is the moment projection ($\mathrm{M}$-projection)~\cite{csiszar1975divergence} of $\mathbb{P}_{Y_t | X_t}$ onto the set $\{M(X_t;\theta) \mid \theta \in \mathbb{R}^d\}$. Based on this, we define the $\mathrm{M}$-projection multi-task objective as follows:
    \[
        \mathcal{L}^{\mathrm{M}}_{\textrm{MT}}(\theta) \triangleq \frac{1}{|\mathcal{T}|} \sum_{t \in \mathcal{T}} \mathcal{H} \left( M(X_t; \theta_t), M(X_t; \theta) \right).
    \]
    Again, by~\autoref{lemma:cross-entropy}, we have
    \[
        \underset{\theta}{\arg\min}~ \mathcal{L}^{\mathrm{M}}_{\textrm{MT}}(\theta) = \underset{\theta}{\arg\min}~ \mathrm{KL} \left( M(X_t; \theta_t) \| M(X_t; \theta) \right).
    \]
    This concludes the proof.
\end{proof}

\subsection{Distribution shift} \label{app:th-res:distribution-shift}

The objective function we proposed in~\autoref{eq:meth-obj} supposed that for each task we have access to the input data distribution denoted as $\Prob_{X_t}$. However, in some cases we can have no access to $\Prob_{X_t}$ but to an approximation of it, denoted as $\Prob_{\tilde{X}_t}$. For example, we have a model trained on sentiment analysis and we do not have access the true data. We can thus use existing data for such task as an approximation. 
We show in the following that we can in fact control the behaviour of our method with respect to the quality of the approximation.

\begin{prop}\label{prop:uniform-conv}
    Considering a set of approximated distribution $\set{\Prob_{\tilde{X}_t}}$, for $\mathrm{D}=\mathrm{JS}$, our method will converge in a uniform way with $\set{\Prob_{\tilde{X}_t}}$.
\end{prop}

\begin{proof}
    We recall that in~\autoref{eq:meth-obj} for a given task $t$ we have the following,
    \[
        \mathrm{D}_{X_t}\left(\theta_t \,\|\, \Phi_n^\Gamma\right) = \int_x \mathrm{D}\left(\theta_t(.|x) \,\|\, \Phi_n^\Gamma(.|x)\right) \Prob_{X_t}(dx).
    \]
    Then,
    \[
        \begin{split}
        \left|\mathrm{D}_{X_t}\left(\theta_t \,\|\, \Phi_n^\Gamma\right) - \mathrm{D}_{\tilde{X}_t}\left(\theta_t \,\|\, \Phi_n^\Gamma\right)\right| &= \left|\int_x \mathrm{D}\left(\theta_t(.|x) \,\|\, \Phi_n^\Gamma(.|x)\right) (\Prob_{X_t}(dx) - \Prob_{\tilde{X}_t}(dx))\right| \\
        &\leq \int_x \mathrm{D}\left(\theta_t(.|x) \,\|\, \Phi_n^\Gamma(.|x)\right) \left|\Prob_{X_t}(dx) - \Prob_{\tilde{X}_t}(dx)\right|
        \end{split}
    \]
    If we use the Jensen Shannon divergence we then have,
    \[
        \left|\mathrm{JS}_{X_t}\left(\theta_t \,\|\, \Phi_n^\Gamma\right) - \mathrm{JS}_{\tilde{X}_t}\left(\theta_t \,\|\, \Phi_n^\Gamma\right)\right| \leq \log(2) \int_x \left|\Prob_{X_t}(dx) - \Prob_{\tilde{X}_t}(dx)\right|
    \]
    Then by Scheffe's Theorem~\cite{scheffeUsefulConvergenceTheorem1947a}, we have:
    \[
        |\mathrm{JS}_{X_t}\left(\theta_t \,\|\, \Phi_n^\Gamma\right) - \mathrm{JS}_{\tilde{X}_t}\left(\theta_t \,\|\, \Phi_n^\Gamma\right)| \leq 2\log(2) \mathrm{TV}(\Prob_{X_t}, \Prob_{\tilde{X}_t}),
    \]
    where $\mathrm{TV}$ stands for total variation distance. Then we have,
    \[
        \left|\sum_t \left(\mathrm{D}_{X_t}\left(\theta_t \,\|\, \Phi_n^\Gamma\right) - \mathrm{D}_{\tilde{X}_t}\left(\theta_t \,\|\, \Phi_n^\Gamma\right)\right)\right| \leq 2 \log(2) \sum_t \mathrm{TV}(\Prob_{X_t}, \Prob_{\tilde{X}_t}),
    \]
    which concludes the proof.
\end{proof}

\begin{rem}
    In~\autoref{prop:uniform-conv} we state that the convergences is uniform in the sense that if the approximations we have converge uniformally to the true distribution {\it i.e.} in the sense of the total variation, then we have a convergence of our objective function.
\end{rem}

\section{Divergence Details}
\subsection{Divergence Between Models on Sequence Outputs} \label{app:divergence_sequence_case}

In this section, we give more details on the computation of a divergence $\mathrm{D}$ between autoregressive models. In this study, we recall that we defined the divergence between two LMs $M_1$ and $M_2$ as following,
\[
    \mathrm{D}_X(M_1 \| M_2) \triangleq \mathbb{E}_{X}\left[\mathrm{D}\big(M_1(\cdot \mid X) \| M_2(\cdot \mid X)\big)\right].
\]
%
%
Since a model prompted with input $x$ generates a sequence of symbols in an auto-regressive set-up, we propose here to detail more the way divergence is computed between models.
For each input $x$ in $X$, we generate the next tokens in a greedy manner using the reference model $M_1$, while storing the softened logits (probability distributions) of the generated tokens at each step, until the end-of-sequence (EOS) token is produced. We then append the generated token sequence $y$ to the original input $x$, forming the extended sequence $x + y$. Next, we forward propagate this extended sequence through the second model $M_2$ once, and obtain the probability distributions of the tokens generated by $M_1$. This allows us to compare the token-level distributions of $M_1$ and $M_2$ on the same generated sequence.
Using this procedure, we can measure the divergence $\mathrm{D}_X(M_1 \| M_2)$ over sequences generated by $M_1$. In a more formal way we propose to compute divergence in a recurrent. Let $x \in \Xspace$ be a sequence, $t \in \mathbb{N}$ be an index, and we suppose that we can sample
\[
    \mathrm{D}\big(M_1(\cdot \mid x, y_{<t}) \| M_2(\cdot \mid x,y_{<t})\big).
\]
Then the final divergence is given by,
\[
    \frac{1}{|\Xspace|} \sum_x \frac{1}{T_x} \sum_{t=1}^{T_x} \mathrm{D}\big(M_1(\cdot \mid x, y_{<t}) \| M_2(\cdot \mid x,y_{<t})\big),
\]
Where for each $x$, $T_x$ is the maximum number of token to generate before the end of sequence token (it can be viewed as some sort of stopping time). For greater accuracy, this calculation should be performed as follows,
\[
    \frac{1}{|\Xspace|} \sum_x \frac{1}{T_x} \sum_{t=1}^{T_x} \sum_{y_{<t} \sim M_1(\cdot \mid x)} \mathrm{D}\big(M_1(\cdot \mid x, y_{<t}) \| M_2(\cdot \mid x,y_{<t})\big),
\]
where the sum over $y_{<t} \sim M_1(\cdot \mid x)$, would correspond to a sampling procedure of sequences of size less than $t$ with respect to the model $M_1(\cdot \mid x)$. In our work, for sake of simplicity we stick to some greedy procedure.

\subsection{Correlation between Divergence Variants and Model Relatedness.} \label{divergence_justification}

In~\autoref{sec:results:divergence-justification}, we proposed an experiment to investigate links between the divergences we used and the notion of performance on the different tasks. 
In~\autoref{fig:kl_heatmap}, we propose the heat map defined by $\mathrm{D}_{X_i}(\theta_i \| \theta_j)$, and in~\autoref{tab:glu_accuracies}, we proposed the matrix of values $\mathrm{PERF}(\theta_j, i)$. Correlations computed in~\autoref{tab:spearman_results} in this study, correspond to correlations compute between rows of~\autoref{fig:kl_heatmap} and rows of~\autoref{tab:glu_accuracies}.

\begin{figure}[tb]
  \centering
  \begin{subfigure}[t]{0.48\textwidth}
    \centering
    \includegraphics[width=\linewidth]{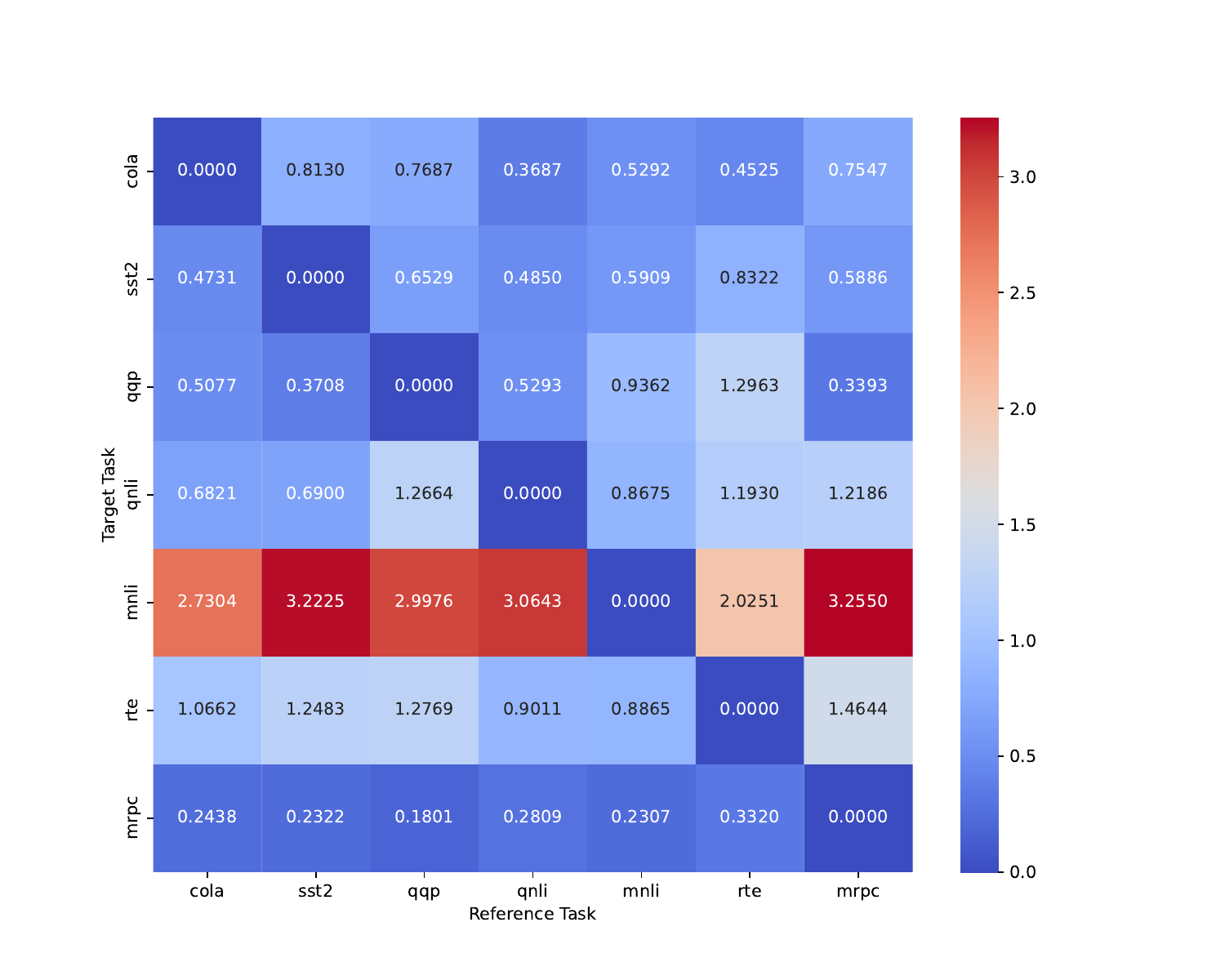}
    \caption{$\mathrm{KL}$}
    \label{fig:kl_heatmap_forward}
  \end{subfigure}
  \hspace{0.01\textwidth}
  \begin{subfigure}[t]{0.48\textwidth}
    \centering
    \includegraphics[width=\linewidth]{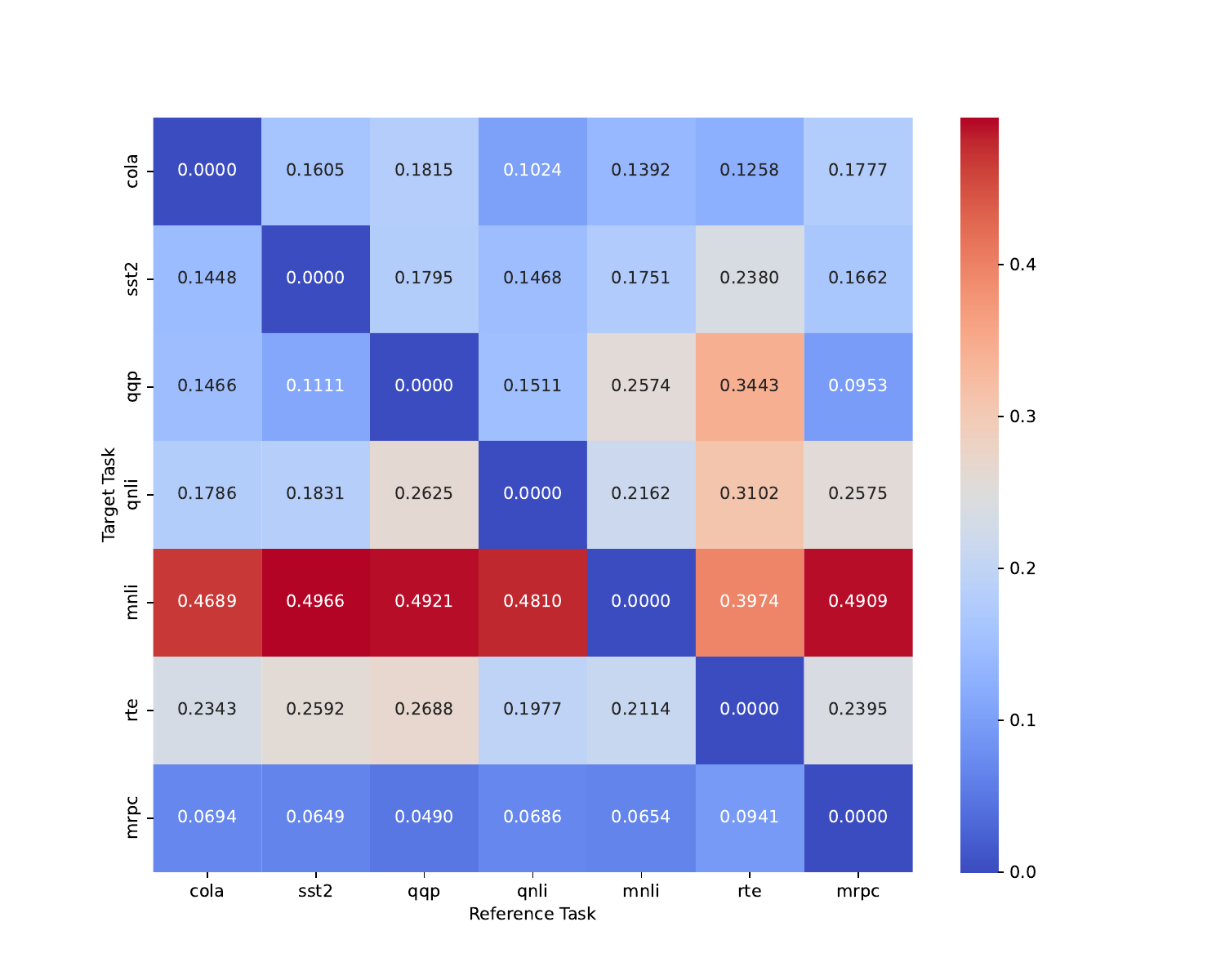}
    \caption{$\mathrm{JS}$}
    \label{fig:kl_heatmap_js}
  \end{subfigure}
  \caption{$\mathrm{D}_{X_i}(\theta_i \| \theta_j)$ values for different divergences. ($i$ corresponds to the row index, while $j$ corresponds to the column index.)}
  \label{fig:kl_heatmap}
\end{figure}

\begin{table}[tb]
\small
\centering
\begin{tabular}{lccccccc}
\toprule
Eval Dataset $\rightarrow$ & \textbf{CoLA} & \textbf{SST-2} & \textbf{QQP} & \textbf{QNLI} & \textbf{MNLI} & \textbf{RTE} & \textbf{MRPC} \\
Model $\downarrow$ & & & & & & & \\
\midrule
\textbf{CoLA} & \textbf{82.20} & 77.80 & 50.60 & 44.60 & 8.80  & 28.88 & 70.00 \\
\textbf{SST-2} & 44.00 & \textbf{92.80} & 71.20 & 37.00 & 10.20 & 26.71 & 66.75 \\
\textbf{QQP}  & 33.00 & 50.20 & \textbf{84.60} & 40.80 & 8.60  & 28.88 & 69.75 \\
\textbf{QNLI} & 46.20 & 76.40 & 39.20 & \textbf{86.40} & 9.20  & 43.68 & 69.00 \\
\textbf{MNLI} & 33.80 & 61.40 & 65.40 & 45.00 & \textbf{78.00} & 28.16 & 67.50 \\
\textbf{RTE}  & 34.80 & 48.20 & 63.60 & 50.20 & 14.00 & \textbf{75.81} & 69.00 \\
\textbf{MRPC} & 32.80 & 57.40 & 66.20 & 41.20 & 11.60 & 45.85 & \textbf{85.00} \\
\bottomrule
\end{tabular}
\caption{
Accuracies (\%) of each model checkpoint (rows) evaluated on the seven GLUE tasks (columns). Each row corresponds to a model fine-tuned on a specific task. The highest accuracy for each task is highlighted in bold, and corresponds each time to the specialized model.
}
\label{tab:glu_accuracies}
\end{table}

\section{Additional experiments}
\subsection{Task Vectors Cosine Similarities}

As is well known in the model merging literature, and more specifically within the task arithmetic framework, cosine similarities between task vectors are typically close to zero. This indicates that the tasks are sufficiently disentangled and can be effectively merged using task arithmetic methods. 
In \autoref{fig:cosine-similarity-matrices}, we present the cosine similarity matrices of the task vectors used in our experiments.

\begin{figure}[tb]
    \centering
    \includegraphics[width=1\linewidth]{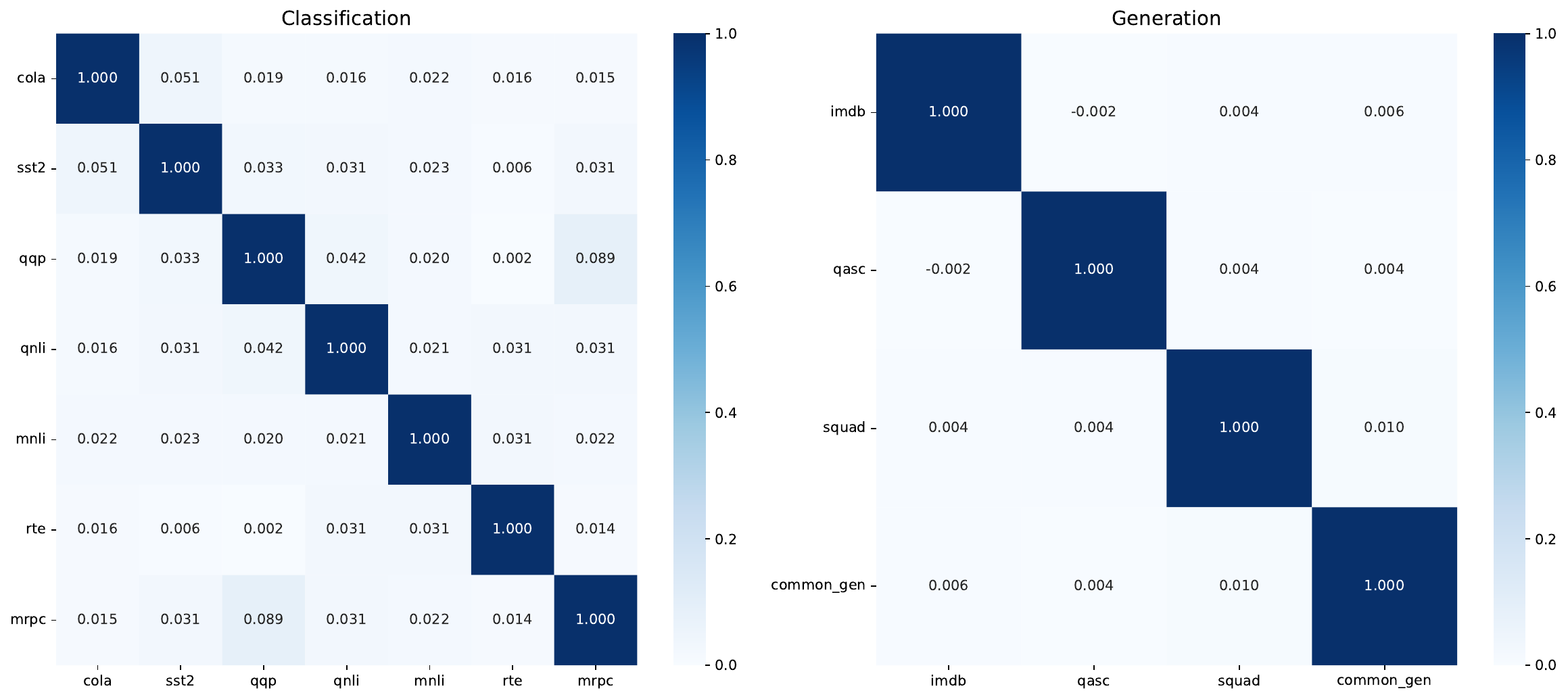}
    \caption{Cosine similarity matrices of task vectors between different tasks. Left: Similarity between GLUE benchmark tasks (CoLA, SST2, QQP, QNLI, MNLI, RTE, MRPC). Right: Similarity between diverse generative tasks tasks (IMDB, QASC, SQUAD, commonGen). Lower similarity values indicate greater orthogonality between task vectors, suggesting less interference when merging models fine-tuned on these tasks.}
    \label{fig:cosine-similarity-matrices}
\end{figure}

\subsection{Details on~\autoref{fig:multi-tasks}} \label{glue_res_T5}

On~\autoref{tab:merged_tasks_glue}, we propose the values that are plotted on~\autoref{subfig:classif}, and on~\autoref{tab:merged_tasks_T5} we propose the values that are plotted on~\autoref{subfig:generation}.

\begin{table}[tb]
\centering
\small
\resizebox{\textwidth}{!}{%
\renewcommand{\arraystretch}{1.2}
\begin{tabular}{c|ccc|ccc|ccc}
\toprule
\textbf{\# Tasks} & \textbf{Model Averaging} & \textbf{Multi-SLERP} & \textbf{TIES} & \multicolumn{3}{c}{\textbf{Task Level}} & \multicolumn{3}{c}{\textbf{Layer Level}} \\
\cmidrule(lr){5-7} \cmidrule(lr){8-10}
& & & & Adamerging & KL (ours) & JS (ours) & Adamerging & KL (ours) & JS (ours) \\
\midrule
2 & 93.89 & 94.20 & 96.02 & 91.16 & 98.18 & 98.28 & 92.41 & 98.85 & 98.85 \\
3 & 89.10 & 92.18 & 94.08 & 90.86 & 97.40 & 97.39 & 90.26 & 98.42 & 98.26 \\
4 & 75.12 & 79.31 & 79.37 & 78.73 & 80.12 & 79.21 & 78.47 & 95.20 & 95.19 \\
5 & 60.92 & 66.86 & 73.61 & 66.84 & 83.82 & 85.37 & 64.93 & 95.60 & 95.50 \\
6 & 56.98 & 70.48 & 67.23 & 67.97 & 83.45 & 84.45 & 62.85 & 93.11 & 93.42 \\
7 & 60.51 & 68.89 & 68.39 & 67.26 & 83.45 & 84.70 & 63.05 & 92.53 & 93.06 \\
\midrule
\textbf{Average} & 72.75 & 78.65 & 79.78 & 77.14 & 87.73 & 88.23 & 75.33 & 95.62 & 95.71 \\
\bottomrule
\end{tabular}%
}
\caption{$\mathrm{ANP}$ for merged tasks obtained via different merging methods. Values are normalized as percentages, with separate evaluations for KL and JS Divergence variants.}
\label{tab:merged_tasks_glue}
\end{table}

\begin{table}[tb]
\centering
\small
\resizebox{\textwidth}{!}{%
\renewcommand{\arraystretch}{1.2}
\begin{tabular}{c|ccc|ccc|ccc}
\toprule
\textbf{\# Tasks} & \textbf{Model Averaging} & \textbf{Multi-SLERP} & \textbf{TIES} & \multicolumn{3}{c}{\textbf{Task Level}} & \multicolumn{3}{c}{\textbf{Layer Level}} \\
\cmidrule(lr){5-7} \cmidrule(lr){8-10}
& & & & Adamerging & Forward (ours) & JS (ours) & Adamerging & Forward (ours) & JS (ours) \\
\midrule
2 & 97.92 & 98.96 & 98.43 & 99.74 & 96.34 & 98.96 & 100.00 & 98.96 & 99.48 \\
3 & 82.62 & 67.89 & 87.87 & 79.92 & 92.99 & 96.36 & 89.26 & 95.91 & 97.79 \\
4 & 53.38 & 64.85 & 56.42 & 60.88 & 87.52 & 91.75 & 83.15 & 94.97 & 97.44 \\
\midrule
\textbf{Average} & 77.97 & 77.23 & 80.91 & 80.18 & 92.28 & 95.69 & 90.80 & 96.61 & 98.24 \\
\bottomrule
\end{tabular}%
}
\caption{$\mathrm{ANP}$ for merged tasks obtained via different merging methods. Values are normalized as percentages, with separate evaluations for KL and JS Divergence variants.}
\label{tab:merged_tasks_T5}
\end{table}

\subsection{Parameter convergence}

In~\autoref{sec:results} we provided an analysis of the convergence of our method by displaying the evolution of our loss function through training iterations and we concluded that our method smoothly converges to an local optimum value. We decided to go further and analyse the evolution of the coefficients associated to each task. As a recall we used the framework of task arithmetic and in this framework the merged model is given by the following,
\[  
    \theta_0 + \sum_i \Gamma_i \times \tau_i,
\]
and we are here interested into the evolution of the coefficients $\Gamma_i$. In~\autoref{fig:alpha_beta_evolution}, we provide the evolution of $\Gamma_1$ (left) and $\Gamma_2$ (right) through training iterations, on different pairwise merging set-up on the benchmark GLUE. We can mainly observe that the dynamic of our method is also smooth in the coefficients $\Gamma_i$, with an interesting convergence of the parameters.
We can also go further by observing in some settings that the values of $\Gamma_1$ and $\Gamma_2$ seem to be independent meaning that when merging two tasks the merging coefficient associated to one task seems to strongly depend on the task itself and not the task with which we merge.
To better support this fact, we decided to add a visualization. In the framework of task arithmetic, each merging experiment can be represented by a point in an euclidean space defined by the following coordinates $\left(\Gamma_1, \Gamma_2, \dots, \Gamma_n\right)^t$. In the case of pairwise merging experiments, these points are in a plan and we decided to visualize this plan on~\autoref{fig:glue_tasks_coefficients}, for classification tasks, and~\autoref{fig:t5_tasks_coefficients} for generative tasks.
On these figures, we can mainly observe that we have different scenarios. For tasks such as QNLI, the factor associated with the QNLI task seems not to depend on the other tasks, while for some other tasks such as MRPC and CoLA we have another scenario where the value of the coefficient associated to the task seems to depend on the value associated to the other tasks.
This seems to be an interesting observation, to be considered alongside the fact that some tasks may be independent, while others may have a statistical dependency, i.e., completing one task may have a positive or negative impact on another.


\begin{figure*}[tb]
    \centering
    \includegraphics[width=1\linewidth,height=0.3\textheight,keepaspectratio]{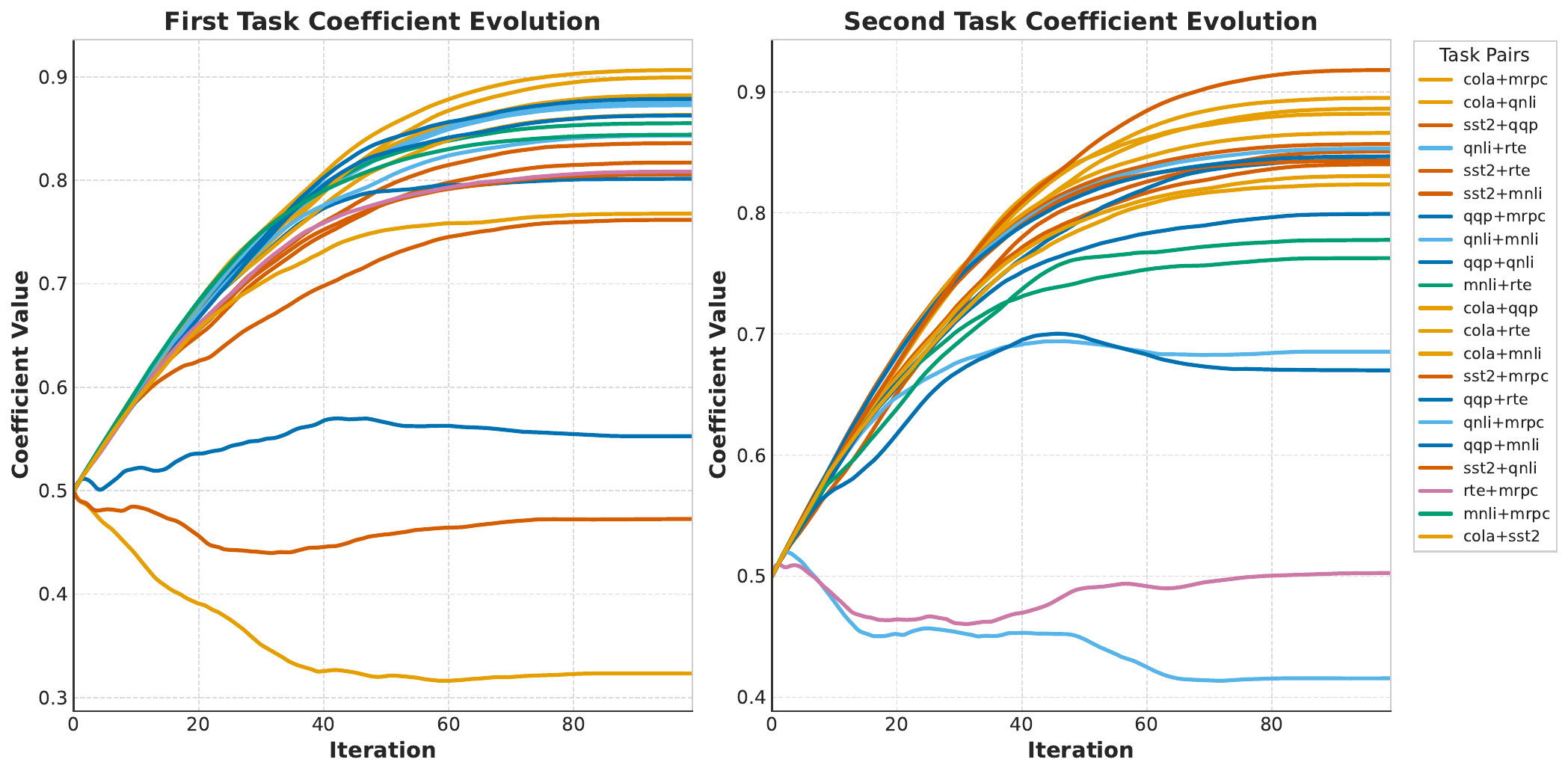}
    \caption{Evolution of the task coefficients across training iterations. The first graph shows the coefficient assigned to the first task in each task pair (as indicated in the legend), while the second graph shows the coefficient assigned to the second task.}
    \label{fig:alpha_beta_evolution}
\end{figure*}



\begin{figure}[tb]
    \centering
    \includegraphics[width=1\linewidth]{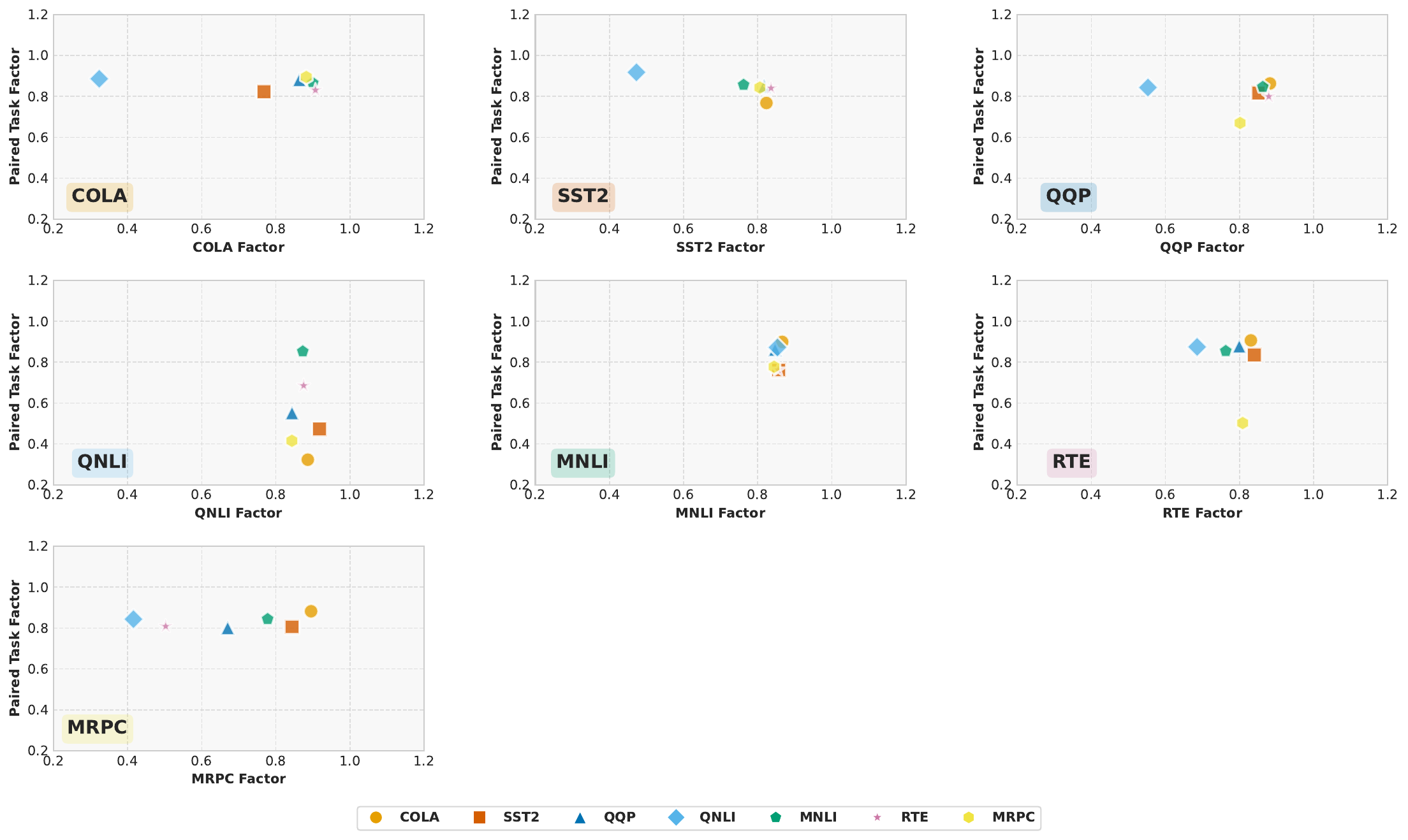}
    \caption{Visualization of coefficient values for a fixed reference task versus coefficient values for the remaining GLUE tasks. Each subplot corresponds to a different reference task.}
    \label{fig:glue_tasks_coefficients}
\end{figure}

\begin{figure}[tb]
    \centering
    \includegraphics[width=1\linewidth]{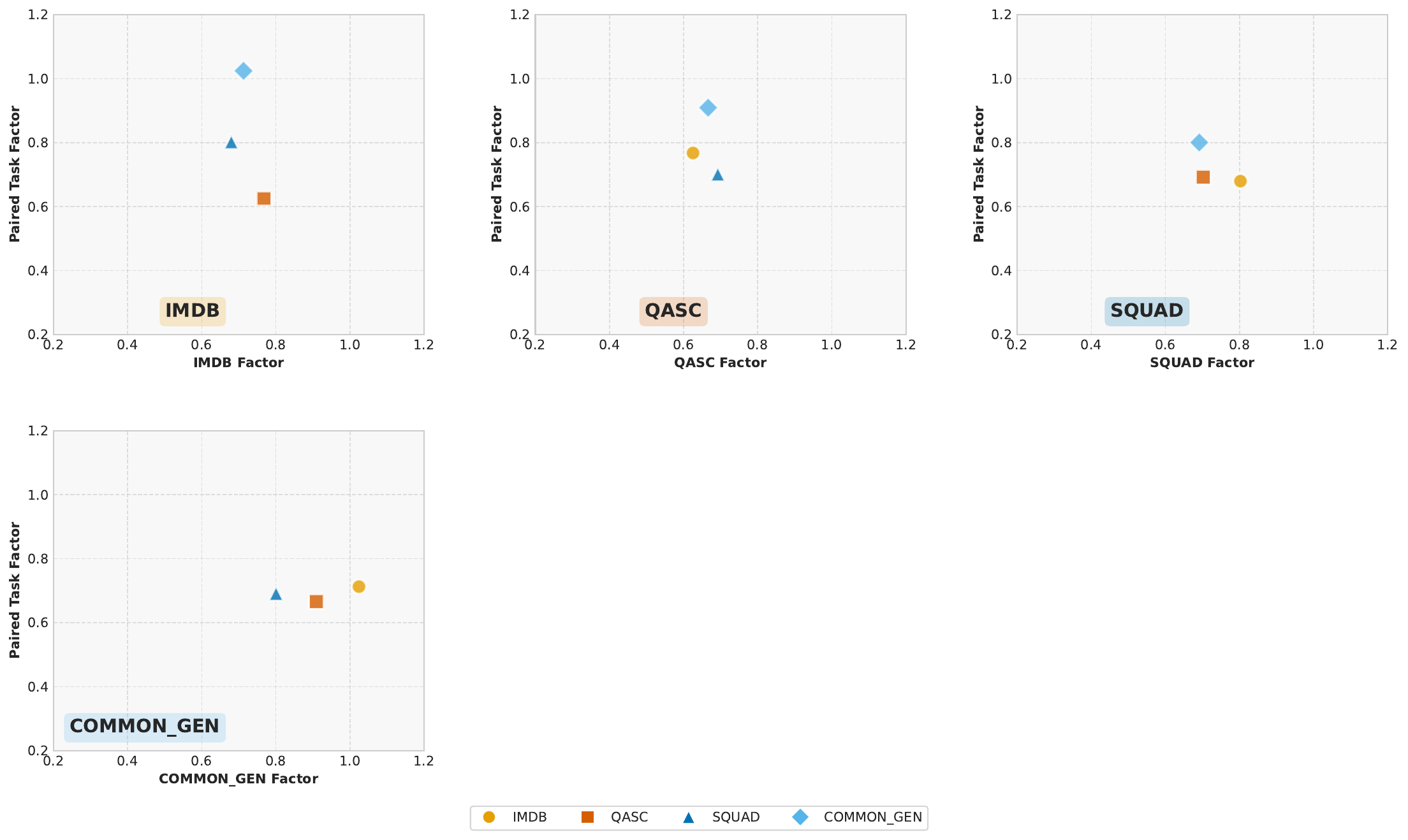}
    \caption{Coefficient values for each task at different T5 checkpoints. Each plot fixes a reference task and compares its coefficient to those of the other tasks.}
    \label{fig:t5_tasks_coefficients}
\end{figure}


\section{Training Settings}
\label{app:training-details}

\subsection{Data details}

As explained in~\autoref{sec:expe-proto}, we used the GLUE Benchmark~\cite{wang2019gluemultitaskbenchmarkanalysis} to perform our experiments. We recall on~\autoref{tab:glue-description} the description of tasks from this benchmark.

\begin{table}[tb]
    \centering
    \begin{tabular}{l|l}
        \toprule
         CoLA & detection of the linguistic acceptability of a sentence\\
         MNLI & natural language inference \\
         MRPC & paraphrase detection\\
         QNLI & question answering converted into natural language inference \\
         QQP & detection of equivalence between questions \\
         RTE & natural language inference \\
         SST2 & sentiment analysis \\
         \bottomrule
    \end{tabular}
    \caption{Description of the GLUE Benchmark}
    \label{tab:glue-description}
\end{table}

\subsection{Training details}

We propose in~\autoref{tab:training-params} training hyper-parameters we chose for our method, as well as for the Adamerging one since it also requires a training procedure. All the optimizations were done using the Adam optimizer~\cite{kingma2014adam} with default moments hyper-parameters.
From a practical standpoint, the hyperparameters we choose, both for the Adamerging method and for our own, allow us to maximize  multi-task performance on evaluation sets. 
While the choice of hyperparameters is relatively sensitive in the Adamerging method, our method appears to be more robust in terms of hyperparameter selection.
The other methods we used that had hyper-parameters were TIES and Multi-SLERP, for which we basically used the recommended recipes:
\begin{itemize}[leftmargin=*]
    \item \textbf{TIES:} We used the recommended recipe from~\cite{yadav2023tiesmergingresolvinginterferencemerging}, with $\lambda = 1$ and a mask rate of 0.2 (i.e., 80\% zeros in the mask).
    \item \textbf{Multi-SLERP:} The weights were set to $1/N$, where $N$ is the number of tasks.
\end{itemize}

\begin{table}[tb]
    \centering
    \begin{tabular}{lcccccc}
    \toprule
    \textbf{Method} & \textbf{Level} & \textbf{Batch Size} & \textbf{Epochs} & \textbf{Dataset Size} & \textbf{Scheduler (LR)} & \textbf{Init. Param.} \\
    \midrule
    $\mathrm{KL}$/$\mathrm{JS}$ (Classif)      & Task   & $4 \times$ \# tasks & 4 & 200 & $1\mathrm{e}{-2}$ & 0.5 \\
    $\mathrm{KL}$/$\mathrm{JS}$ (Classif)     & Layer  & $4 \times$ \# tasks & 4 & 400 & $1\mathrm{e}{-2}$ & 0.5 \\
    \hline
    $\mathrm{KL}$/$\mathrm{JS}$ (Gen.)      & Task   & $4 \times$ \# tasks & 4 & 200 & $1\mathrm{e}{-2}$ & 0.5 \\
    $\mathrm{KL}$/$\mathrm{JS}$ (Gen.)     & Layer  & $4 \times$ \# tasks & 4 & 400 & $1\mathrm{e}{-2}$ & 0.5 \\
    \hline \hline
    AdaMerging (Classif)     & Task   & $4 \times$ \# tasks & 5 & 200 & $1\mathrm{e}{-3}$ & 0.5 \\
    AdaMerging (Classif)     & Layer  & $4 \times$ \# tasks & 5 & 400 & $1\mathrm{e}{-3}$ & 0.5 \\
    \hline
    AdaMerging (Gen.)     & Task   & $4 \times$ \# tasks & 5 & 200 & $1\mathrm{e}{-2}$ & 0.5 \\
    AdaMerging (Gen.)    & Layer  & $4 \times$ \# tasks & 5 & 400 & $1\mathrm{e}{-2}$ & 0.5 \\
    \bottomrule
    \end{tabular}
    \caption{Training configurations.}
    \label{tab:training-params}
\end{table}

\section{Everything is task arithmetic}
\label{app:everything-is-ta}

Many different merging methods have emerged in the landscape of machine learning. Among them, task arithmetic~\cite{ilharco2022editing} is probably the most widely used. As a reminder, the merging function in the case of task arithmetic is defined as follows:
\[
    f\left(\theta_0, \set{\tau_t}, \Gamma \right) = \theta_0 + \sum_t \Gamma_t \times \tau_t.
\]
An interesting question that naturally arises is the following: Given a merging method $g\left(\theta_0, \{\tau_t\}, \Delta \right)$, can we find coefficients $\Gamma$ such that
$
    g\left(\theta_0, \{\tau_t\}, \Delta \right) = f\left(\theta_0, \{\tau_t\}, \Gamma \right)?
$
\\
If this is true, then we can state that
\[
\min_{\Gamma}~\sum_t \mathrm{D}_{X_t}\left(\theta_t \,\|\, f(\theta_0, \{\tau_t\}, \Gamma)\right) \leq \min_{\Delta}~\sum_t \mathrm{D}_{X_t}\left(\theta_t \,\|\, g(\theta_0, \{\tau_t\}, \Delta)\right).
\]
This inequality would highlight the strength of our method, as it encompasses the entire range of task arithmetic. As stated in~\autoref{rem:hp-est}, our method can be applied to any hyper-parameter differentiable merging approach.
As pointed out in~\cite{goddard2024arcee}, a wide range of merging methods are based on task arithmetic, with the main differences lying in the estimation of the merging coefficients. In the following, we provide an analysis of other merging methods to show that they can be expressed as model merging. This demonstrates that our method has the potential to achieve better results.

\paragraph{SLERP.} Spherical linear interpolation (SLERP)~\cite{wortsman2022modelsoupsaveragingweights} is a classical method used to combine vectors on a spherical manifold. For this method, we introduce a hyperparameter $t \in [0,1]$, and define SLERP as follows:
\[
    f(\theta_0, \{\tau_1, \tau_2\}, t) \triangleq \theta_0 + \frac{\sin((1-t)\Omega)}{\sin \Omega}\frac{\tau_1}{\|\tau_1\|} + \frac{\sin (t\Omega)}{\sin \Omega}\frac{\tau_2}{\|\tau_2\|},
\]
where $\Omega$ is the angle between $\tau_1$ and $\tau_2$.

\paragraph{Fisher Weight Averaging.} The Fisher weight averaging method, introduced in~\cite{matena2022merging}, is a merging technique that reduces to task arithmetic in the case of linear interpolation between specialized models. The merging coefficients for each model are based on the Fisher Information matrix and are determined by solving an optimization problem related to finding a centroid between models. 
One limitation, as pointed out in the original paper, is the high computational cost of estimating the Fisher Information matrix to obtain the merging coefficients. This estimation can also be numerically unstable, as the coefficients in the matrix can be close to zero. 
Additionally, this method introduces extra scaling hyperparameters that must be tuned.

\paragraph{RegMean.} The RegMean merging method, proposed in~\cite{jin2022dataless}, also reduces to task arithmetic, as it performs a linear interpolation between specialized models. This interpolation aims to minimize the $L_2$ distance between the merged model and the individual models, whereas our method is designed to minimize the $\mathrm{JS}$ (or $\mathrm{KL}$) divergence between models. The $L_2$ distance is a restrictive measure. 
Moreover, as stated in~\cite{blau2018perception,blau2019rethinking,zhang2023rate}, $L_2$ distance is a distortion measure, while $\mathrm{KL}$ and $\mathrm{JS}$ are perception measures. Minimizing perception distance appears to be more suitable for downstream applications, such as performing other tasks.

\paragraph{Kracher Mean.} The Kracher mean (or Riemannian centroid), originally formulated in~\cite{grove1973conjugate} can be used as a merging method which consists in finding some sort of centroid of a finite set of task vectors, denoted as $\set{\tau_t}$. To do so, we suppose that task vectors lies in a Finite dimension Hilbert Space $(H, <\cdot,\cdot>)$, where $<\cdot,\cdot>$ is the standard dot product onto this space and thus $\|\cdot\|$ is the associated norm. The Kracher mean is defined as following,
\[
    \tau_{\textrm{F}} \triangleq \arg \min_{\tau \in H} \sum_t \|\tau - \tau_t\|^2.
\]
The following proposition holds,
\begin{prop}\label{prop:kracher}
    Let $(H, <\cdot, \cdot>)$ be a finite dimension Hilbert space. 
    Then for all set of point $\set{\tau_t} \subset H$, representing task vectors, the solution of the Kracher mean (or equivalently the centroid) can be expressed in the task arithmetic framework.
\end{prop}

\begin{proof}
    Let $\set{\tau_t} \subset H$. Let $F \triangleq \mathrm{Span}\left(\set{\tau_t}\right)$. Let $\tau \in H$. We have the following result,
    \[
        \tau = p_1 + p_2,~\text{s.t.}~p_1\in F, p_2 \in F^\perp.
    \]
    Then we have,
    \[
        \begin{split}
            \psi(\tau) &\triangleq \sum_t \|\tau - \tau_t\|^2, \\
            &= \sum_t <\tau - \tau_t, \tau - \tau_t>, \\
            &= \sum_t \|\tau\|^2 - 2<\tau, \tau_t> + \|\tau_t\|^2, \\
            &= \sum_t \|p_1\|^2 + \|p_2\|^2 - 2<p_1, \tau_t> + \|\tau_t\|^2.
        \end{split}
    \]
    Then by taking, $\tau^{\prime} = p_1$, we have $\psi(\tau^{\prime})\leq \psi(\tau)$, which leads to the following statement: $\forall \tau \in H$, $\exists \tau^\prime \in F$, such that,
    \[
        \psi(\tau^{\prime}) \leq \psi(\tau).
    \]
    Then $\arg \min_{\tau \in H} \psi(\tau) \in F$, which concludes the proof.
\end{proof}

\begin{rem}\label{rem:hilbert-dist}
    As stated in~\autoref{rem:centroid}, the method we proposed in this study can also be viewed as a centroid. However the framework we used does not allow to connect directly to the theory of Kracher mean. In fact, if one would want to formulate our method as a Kracher mean, the "distance" defined over the space of task vectors would be the following,
    \[
        d(\tau_i, \tau_j) = \mathrm{D}_{X_i}\left(\theta_i \| \theta_j \right).
    \]
    However, even in the case of the Jensen Shannon divergence this "distance" is not a mathematical one as it does not respect the property of the distance. Consequently it does not define a metric space and therefore even less a Hilbert space. 
    Then an interesting line of research would be to identify the possible distances one could define over the space of task vectors. From the result we just demonstrated, if we can verify that the distance can be derived from a dot product and thus induce a Hilbert Space, then we can conclude that the optimal solution lies in the framework of task arithmetic, giving thus added weight to this method and possibly offering more theoretical explanations as to why this method is in many cases the state of the art.
\end{rem}

\end{document}